\documentclass[final,onefignum,onetabnum]{siamonline220329}

\usepackage{subcaption}
\usepackage{amssymb, mathtools}
\usepackage{algorithm}
\usepackage[noend]{algorithmic}
\usepackage{xcolor}
\usepackage{stmaryrd}
\usepackage{bbm}
\usepackage{enumitem}

\newsiamremark{assumption}{Assumption}

\newcommand{\eps}{\varepsilon}
\newcommand{\norm}[1]{\left\lVert#1\right\rVert}
\newcommand{\abs}[1]{\left\lvert#1\right\rvert}
\newcommand{\R}{\mathbb{R}}
\newcommand{\N}{\mathbb{N}}

\newcommand{\E}{\mathbb{E}}

\newcommand{\cA}{\mathcal{A}}
\newcommand{\cN}{\mathcal{N}}
\newcommand{\cX}{\mathcal{X}}

\newcommand{\Lip}{{\rm Lip}}
\newcommand{\sep}{:}
\newcommand{\vol}{\textrm{vol}}

\newcommand{\errt}[1]{{\rm err}_\tau(#1, E)}
\newcommand{\errtzero}{{\rm err}_\tau(A, E^*)}

\newcommand{\Ni}[1]{\mathcal{N}\left(\mathcal{X}_{(\eps_{#1}, \eps_{#1-1}]}, \frac{\eps_{#1}}{L}\right)}
\newcommand{\Neps}{\mathcal{N}\left(\mathcal{X}_\eps, \frac{\eps}{L}\right)}
\newcommand{\sprec}{\sigma(A, E, \eps)}
\newcommand{\sprecAE}[2]{\sigma(#1, #2, \eps)}
\newcommand{\cinf}{C_{\inf}(f, \eps)}
\newcommand{\teps}{\tilde{\eps}}
\newcommand{\mfdoo}{\text{c.MF-DOO}}
\newcommand{\mfsoo}{\text{c.MF-StoOO}}
\newcommand{\diam}{{\rm diam}}
\newcommand{\defeq}{\vcentcolon =}

\DeclareMathOperator*{\argmax}{arg\,max}

\title{Certified Multi-Fidelity Zeroth-order Optimization}
\author{\'{E}tienne de Montbrun, S\'{e}bastien Gerchinovitz}
\date{\today}

\ifpdf
  \DeclareGraphicsExtensions{.eps,.pdf,.png,.jpg}
\else
  \DeclareGraphicsExtensions{.eps}
\fi

\usepackage{enumitem}
\setlist[enumerate]{leftmargin=.5in}
\setlist[itemize]{leftmargin=.5in}

\newsiamremark{remark}{Remark}
\newsiamremark{hypothesis}{Hypothesis}
\crefname{hypothesis}{Hypothesis}{Hypotheses}
\newsiamthm{claim}{Claim}

\headers{Certified Multi-Fidelity Zeroth-order Optimization}{Etienne de Montbrun, Sébastien Gerchinovitz}

\title{Certified Multi-Fidelity Zeroth-order Optimization}

\author{Étienne de Montbrun\thanks{ANITI \& Toulouse School of Economics, France
  (\email{edemontb@ens-paris-saclay.fr})}
\and Sébastien Gerchinovitz\thanks{IRT Saint Exupéry \& Institut de Mathématiques de Toulouse, France
  (\email{sebastien.gerchinovitz@irt-saintexupery.com})}
}

\ifpdf
\hypersetup{
  pdftitle={Certified Multi-Fidelity Zeroth-order Optimization},
  pdfauthor={Étienne de Montbrun, Sébastien Gerchinovitz}
}
\fi

\begin{document}

\maketitle

\begin{abstract}
We consider the problem of multi-fidelity zeroth-order optimization, where one can evaluate a function $f$ at various approximation levels (of varying costs), and the goal is to optimize $f$ with the cheapest evaluations possible. In this paper, we study \emph{certified} algorithms, which are additionally required to output a data-driven upper bound on the optimization error. We first formalize the problem in terms of a min-max game between an algorithm and an evaluation environment. We then propose a certified variant of the MFDOO algorithm and derive a bound on its cost complexity for any Lipschitz function $f$. We also prove an $f$-dependent lower bound showing that this algorithm has a near-optimal cost complexity. As a direct example, we close the paper by addressing the special case of noisy (stochastic) evaluations, which corresponds to $\eps$-best arm identification in Lipschitz bandits with continuously many arms.\footnote{Published in SIAM/ASA Journal on Uncertainty Quantification 12(4), pp. 1135--1164, 2024. The published version can be accessed at \url{https://doi.org/10.1137/23M1591086}.}
\end{abstract}



\section{Introduction}

We consider the problem of multi-fidelity zeroth-order optimization, which unfolds roughly as follows (details are given in Section~\ref{sec:setting}). Let $f:\cX \subset \R^d\to \R$ be a function. Assume that at any $x \in \mathcal{X}$, we can query the value $f(x)$ with any desired accuracy $\alpha>0$, at a cost of $c(\alpha)$. Accurate evaluations (small $\alpha$) come at a high cost. The goal of multi-fidelity optimization is to maximize $f$ with the cheapest evaluations possible.

A typical example is the optimization of a function $f$ computed with finite element modeling.
A case with two fidelity functions (two values of $\alpha$) appears in \cite{sun2011multi}, for sheet-metal forming design with the goal of having no defects in the products (automobile inner panel in that paper).
Given three variables $x_1, x_2, x_3$ modeling strong restraining forces on the metal, the goal is to set this forces to a good value to avoid both rupture and wrinkling.
Two different finite element solvers were used to approximate $f$ at any point $x$: incremental finite element solvers, or a one-step finite element model, which is computationally cheap but provides worse estimates than the former model.
Finding a good design of the forces at a reasonable computational cost is an example of multi (two) fidelity optimization problem.
Many other examples can be found, e.g., in thermodynamics \cite{dewettinck1999modeling,gratiet2013multi}, design of new aircraft \cite{geiselhart2011integration}, or nuclear criticality safety \cite{picheny2010noisy}.

\paragraph{Certified optimization} In practice, algorithms that achieve small optimization errors with small evaluation costs are desirable but may not inform the user when a small optimization error has been obtained. In the example above, an engineer might require to \emph{certify} the output of the algorithm, that is, to get a guaranteed optimization error bound that they can compute by only using the observed data and some (light) prior knowledge on~$f$, as is done, e.g., in \cite{HansenETAL-92-Survey,bachoc2021instance} in the single-fidelity setting. Such requirement can be important in industrial fields involving safety-critical systems (e.g., cars, aircraft, health, nuclear engineering).

In this paper, we study the problem of finding a \emph{certifiably} approximate-maximizer of a Lipschitz function $f$ in the multi-fidelity setting. We quantify the smallest evaluation cost to reach this goal, by deriving nearly-matching upper and lower bounds for any such $f$.

\subsection{Setting}
\label{sec:setting}

We now formally define the setting. Let $\mathcal{X} \subset \R^d$ be a non-empty\footnote{All throughout the paper, $\cX$ is implictly assumed to be non empty.} compact set endowed with a norm $\norm{\cdot}$, and $f:\cX\to \R$ be an $L$-Lipschitz function, with a maximizer $x^\star \in \cX$.\footnote{In fact, all results of Section \ref{sec:mfdoo} still hold if $f$ is simply $L$-Lipschitz around $x^\star$, that is, if $f(x) \geq f(x^\star) - L \norm{x-x^\star}$ for all $x \in \cX$.} Let also $c:(0,+\infty) \to [0,+\infty)$ be a non-increasing cost function.

The problem, which we describe in the online protocol below, can be seen as an interaction between two players: the algorithm $A$ whose goal is to maximize $f$, and an environment $E$ which returns perturbed values of $f$. They interact together in the following way: at every round $t \geq 1$, $A$ picks a query point $x_t \in \cX$ and an evaluation accuracy $\alpha_t>0$; it then observes $y_t = E_t(x_t,\alpha_t) \in [f(x_t) - \alpha_t, f(x_t) + \alpha_t]$ at a cost of $c(\alpha_t)$ ($E_t(\cdot, \alpha)$ is sometimes called $\alpha$-fidelity function); finally $A$ recommends some candidate $x_t^* \in \mathcal{X}$ for a maximizer of $f$, and outputs an error certificate $\xi_t \geq 0$ with the constraint that $\xi_t  \geq \max_{x \in \mathcal{X}} f(x) - f(x_t^*)$ (see a more formal definition below). This way, when using the algorithm, we can not only find an $\eps$-maximizer $x_t^*$, but we can \emph{know} when it is $\eps$-optimal by looking at $\xi_t$, and thus confidently stop searching.

Note that the case of the constant cost $c(\alpha)=1$ for all $\alpha>0$ can be reduced to the single-fidelity setting, where $f$ is observed perfectly (see Appendix~\ref{sec:specialcase-constantcost} for details).

\begin{algorithm}
  \caption*{\textbf{Online Protocol:} Certified multi-fidelity zeroth-order optimization}
  \label{algo:accuracy}
  \mbox{\textbf{Init}: The environment secretly observes $f$ and picks $E = (E_t)_{t\geq 1}$ with $E_t(x,\alpha) \in \bigl[f(x)-\alpha,$}
  \mbox{$f(x)+\alpha\bigr]$ for all $t \geq 1$, $x \in \cX$, and $\alpha>0$ (we also call $E$ \emph{the environment})}\ \\[-0.5cm]
  \begin{algorithmic}[1]
    \FOR{ $t=1,2,\ldots,$}
      \STATE $A$ chooses a query point $x_t \in \mathcal{X}$ and an evaluation accuracy $\alpha_t > 0$
      \STATE $A$ incurs a cost $c(\alpha_t)$
      \STATE $E$ returns $y_t = E_t(x_t,\alpha_t) \in [f(x_t) - \alpha_t, f(x_t) + \alpha_t]$ (inaccurate evaluation of $f(x_t)$)
      \STATE $A$ returns a recommendation $x_t^*$ for the maximum of $f$, with an error certificate $\xi_t \geq 0$
      \ENDFOR
  \end{algorithmic}
\end{algorithm}

\noindent
Next we introduce key definitions before describing the optimization goal, our contributions, related works, and useful notation.

\paragraph{Definitions: environments and certified algorithms}
For any $L$-Lipschitz function $f:\mathcal{X} \to \R$, we define the set $\mathcal{E}(f)$ of all environments associated with $f$, which are sequences of functions $E = (E_t)_{t\geq 1}$ with $E_t(x,\alpha) \in \bigl[f(x)-\alpha,f(x)+\alpha\bigr]$ for all $t \geq 1$, $x \in \cX$ and $\alpha>0$. We assume for simplicity that the sequence $E \in \mathcal{E}(f)$ is fixed from the beginning of the online protocol.\footnote{Since we only consider deterministic algorithms $A$ and work towards guarantees that hold uniformly over all environments, this is in fact equivalent to playing against adversarial environments.}

We can now formally define \emph{certified algorithms}. As can be seen from the online protocol above, $x_t$ and $\alpha_t$ are deterministic functions of the past observations $y_1,\ldots,y_{t-1}$, while $x^*_t$ and $\xi_t$ are deterministic functions of $y_1,\ldots,y_{t}$. We additionally require that the certificates $\xi_t$ satisfy $\xi_t \geq \max_{x \in \mathcal{X}} f(x) - f(x^*_t)$ for all rounds $t \in \N^*$, all $L$-Lipschitz functions $f:\mathcal{X} \to \R$ and all environments $E \in \mathcal{E}(f)$.
We call \emph{certified algorithm} any such sequence of functions $A = \bigl(x_t(\cdot), \alpha_t(\cdot),x^*_t(\cdot),\xi_t(\cdot)\bigr)_{t \geq 1}$, and let $\cA$ denote the set of all certified algorithms. With a slight abuse of notation, we also sometimes write  $x_t(E)$, $\alpha_t(E)$, $x^*_t(E)$ and $\xi_t(E)$ to make the dependency on $E$ more explicit.

\paragraph{Optimization goal} Recall that $x^\star \in \cX$ denotes a maximizer of $f$. A classic goal in multi-fidelity optimization is to reach a small optimization error $f(x^\star) - f(x^*_t)$ while minimizing the total cost $\sum_{s=1}^t c(\alpha_s)$ (see, e.g., \cite{sen2018multi,sen2019noisy,fiegel2020adaptive}). In this paper, we address the stronger goal of finding a recommendation $x^*_t$ with an error certificate $\xi_t$ below $\eps$ (and thus an optimization error \emph{known} to be bounded by $\eps$), with the smallest cumulative cost $\sum_{s=1}^t c(\alpha_s)$ possible.

More formally, for any environment $E \in \mathcal{E}(f)$, we define the \emph{cost complexity} $\sprec$ as the smallest total cost for which $A$ can output a certificate below $\eps$ (when run against $E$). It can be expressed as follows (by convention, $\sprec = + \infty$ if no such $C$ exists):
\begin{equation}
\sprec = \inf \left\{ C \in \R \sep \exists \tau \in \N^*, \sum_{t=1}^\tau c(\alpha_t(E)) \leq C \text{ and } \xi_\tau(E) \leq \eps \right\} .
\label{eq:defcostcomplexity}
\end{equation}

\noindent
Equivalently, since costs are nonnegative, $\sprec$ is equal to the total cost incurred by $A$ (when run against $E$) until its certificate~$\xi_t$ falls below $\eps$ for the first time.

In this paper, we are interested in certified algorithms with small cost complexity against any environment, that is, in algorithms $A \in \mathcal{A}$ that approximately reach the infimum
\[
\inf_{A \in \mathcal{A}} \sup_{E \in \mathcal{E}(f)} \sprec
\]
for any unknown $L$-Lipschitz function $f$.
Importantly, the above min-max quantity depends on $f$ (through the set $\mathcal{E}(f)$), since some functions are easier than others to maximize with certified algorithms.

\subsection{Main contributions and outline of the paper}

In this paper, we study the cost complexity of certified algorithms to maximize $L$-Lipschitz functions in the multi-fidelity setting. We prove nearly-matching $f$-dependent upper and lower bounds, which extend the single-fidelity results of \cite{bachoc2021instance}. More precisely, we make the following set of contributions.\\[-0.2cm]

\begin{itemize}
    \item On the modeling side, we formalize the problem of \emph{certified} multi-fidelity zeroth-order optimization (see Section~\ref{sec:setting} above).
    \item In Section~\ref{sec:mfdoo} we define a certified variant of the MFDOO algorithm \cite{sen2018multi}. We bound its cost complexity in terms of the key quantity $S_{\beta, L}(f,\eps)$ defined in \eqref{eq:Sbeta} below.
    \item In Section~\ref{sec:lower_bound} we derive a nearly-matching $f$-dependent lower bound that holds for all certified algorithms (i.e., we derive a lower bound on the min-max quantity above).
    \item Finally, in Section~\ref{sec:stochastic} we address the special case of noisy evaluations of $f$, which corresponds to $\eps$-best arm identification in Lipschitz bandits with continuously many arms. By a simple reduction to the (deterministic) multi-fidelity setting above, we derive a high-probability sample complexity bound which solves a conjecture by \cite{bachoc2021instance}.
\end{itemize}

\ \\
In the appendix we collect some properties and examples (Appendix~\ref{sec:propertiesExamples}), technical proofs (Appendix~\ref{sec:missing_proofs}), useful simple geometric lemmas (Appendix~\ref{sec:proofs}), together with a formal intuitive reduction from the single- to the multi-fidelity setting (Appendix~\ref{sec:specialcase-constantcost}).

\subsection{Related works}
\label{sec:relatedworks}

This paper has connections with several rich literatures, about single- or multi-fidelity settings, within deterministic or Bayesian frameworks, with or without certificates. Next we provide a (non-comprehensive) subset of references to these related works.

In the single-fidelity setting (where the unknown function $f$ can be evaluated perfectly), the problem reduces to zeroth-order global optimization.
Optimization algorithms \emph{without certificates} have a very long history, in convex optimization (e.g., \cite{Nes-04-ConvexOptimization,BoVa-04-ConvexOptimization,Bub-15-ConvexOptimization} for first-order methods and beyond, or \cite{larson19-derivativeFreeOptimizationMethods} for zeroth-order methods), non-convex optimization (e.g., \cite{HansenETAL-92-Survey,HansenETAL-92-NewAlgorithms,JaKa-17-NonConvexOptimizationML,larson19-derivativeFreeOptimizationMethods}), Bayesian optimization (e.g., \cite{garnett2023bayesian}), stochastic optimization (e.g., \cite{Spa-03-StochasticSearchOptimization,bonnans19-convexStochasticOptimization}), or bandit optimization (e.g.,  \cite{munos2014bandits,slivkins2019introduction}). Among the algorithmic techniques that are closest to this paper, we can mention the Piyavskii-Shubert algorithm \cite{piyavskii1972algorithm,shubert1972sequential} for zeroth-order Lipschitz optimization, as well as discretized variants such as the branch-and-bound algorithm of \cite{perevozchikov1990complexity} or the DIRECT algorithm by \cite{jones1993lipschitzian}, to name a few. More recently several variants were also derived in the bandit community, since zeroth-order optimization can be cast as a bandit problem with continuously many arms (or \emph{$\cX$-armed bandits}). Examples of such bandit algorithms for perfect or noisy (stochastic) evaluations of $f$ include DOO \cite{munos2011optimistic}, HOO \cite{bubeck2011x}, (Sto)SOO \cite{munos2011optimistic,valko2013stochastic}, POO \cite{grill2015black}, which are all based on a hierarchical partition of the input domain, as the \mfdoo{} algorithm of the present paper. Another discretization approach (yet computationally more challenging) is the Zooming algorithm for general metric spaces \cite{kleinberg2008multi,kleinberg2019bandits}. We refer the reader to \cite{munos2014bandits,slivkins2019introduction} for further details and references on bandit algorithms for continuously many arms.

Also close to our work is the Bayesian optimization literature, because of its rich contributions to the multi-fidelity setting. For single-fidelity optimization, we can mention the seminal work of Kushner \cite{kushner1964new}, the use of the expected improvement function introduced in \cite{mockus75-bayesianMethodsExtremum}, and the EGO algorithm of \cite{jones1998efficient}, together with convergence rates in \cite{bull2011convergence}. The kriging community also addressed global optimization with noisy observations \cite{forrester2006design,picheny2013benchmark}. In \cite{srinivas2009gaussian,SrinivasETAL-12-GPbandits} the authors design and study a Gaussian Process-based bandit algorithm (GP-UCB) and derive regret bounds both under Bayesian and deterministic assumptions on the underlying function~$f$.
This algorithm was later adapted in \cite{contal2013parallel} for the case of sequential mini-batch queries in parallel computing.
A detailed review of the Bayesian optimization literature can be found in \cite{garnett2023bayesian}. 

We now focus on two key features of our setting: multi-fidelity and error certificates.

\paragraph{Multi-fidelity optimization}
Multiple works in the multi-fidelity setting come from the Bayesian optimization literature. For instance, Huang \cite{huang2006sequential} designed a multi-fidelity variant of the EGO algorithm; a multi-fidelity counterpart of GP-UCB (MF-GP-UCB) was studied in \cite{kandasamy2016gaussian,kandasamy2019multi}, while co-kriging (a multi-fidelity extension of kriging) was studied, e.g., in \cite{forrester2007multi,qian08-bayesianHierarchicalModeling}.
A framework where the (finitely many) fidelity functions may be mutually dependent was introduced in \cite{song2019general}, alongside a high probability guarantee.
A more complete review on multi-fidelity Bayesian optimization can be found in \cite{peherstorfer2018survey} and in \cite[Section 11.5]{garnett2023bayesian}.

Recently bandit algorithms were also extended to the multi-fidelity setting. 
Starting with the case of finitely-many arms \cite{kandasamy2016multi}, several algorithms were designed for continuously-many arms, which corresponds to multi-fidelity zeroth-order optimization \emph{without} certificates. An extension of DOO (called MFDOO) was provided by \cite{sen2018multi}, together with an optimization error bound (or simple regret bound) for a given overall cost budget. The authors also study a variant inspired from POO (MFPDOO) to handle the case of unknown smoothness (see a paragraph below). 
Similar algorithms based on HOO and POO (called MFHOO and MFPOO) were later introduced by \cite{sen2019noisy} to cope with additional (stochastic) noise.
In \cite{fiegel2020adaptive} the authors develop the Kometo algorithm (based on StroquOOL) and prove nearly optimal upper bounds (with matching minimax lower bounds) on the optimization error given an overall cost budget. Their analysis also covers the cases of unknown smoothness and of possibly unbounded costs at accuracies $\alpha$ in the neighborhood of $\alpha = 0$. The case of delayed and noisy feedback was addressed in \cite{wang2022procrastinated} with a generalization of HOO.

\paragraph{Error certificates}
Though convergence results about the optimization error $f(x^\star)-f(x^*_t)$ in terms of a total cost budget are now well established in the multi-fidelity setting, the question of \emph{certifying} approximate maximizers with minimum total cost has not been addressed, to the best of our knowledge.

The notion of error certificate appeared in several other settings, such as, e.g., in convex optimization \cite{BoVa-04-ConvexOptimization}, where the \emph{duality gap} between primal and dual feasible points plays the role of an error certificate. In zeroth-order Lipschitz optimization with perfect evaluations of $f$ (single-fidelity setting), the Piyavskii-Shubert algorithm \cite{piyavskii1972algorithm,shubert1972sequential} is naturally endowed with an error certificate, which is the difference between the maximum value of a guaranteed upper bounding function of $f$ and the maximum value $f(x_s)$ observed so far. For one-dimensional inputs, a tight analysis of the number of evaluations before which this certificate falls below~$\eps$ was given in \cite{HansenETAL-91-NumberIterationsPiyavskii} (see also \cite{danilin71-estimationEfficiencyAbsoluteMinimumFinding}), with a simple integral expression. This result was generalized to multi-dimensional inputs by \cite{bouttier2020regret,bachoc2021instance}, and shown in \cite{bachoc2021instance} to be achievable with the tractable\footnote{Tractability refers to a small (logarithmic) number of elementary operations per evaluation of $f$.} c.DOO algorithm, with a nearly matching $f$-dependent lower bound. In this paper we extend the complexity analysis of \cite{bachoc2021instance} to the multi-fidelity setting, using a certified variant of MFDOO \cite{sen2018multi} for the upper bound.

Another problem that involves the question of certifying an optimization algorithm's output is \emph{best arm identification in stochastic bandits} (e.g., \cite{evendar02-PACboundsMAB,GK16colt-OptimalBAI,degenne19-nonAsymptoticPureExplorationSolvingGames} and \cite[Chapter~33]{LS20-banditalgos}). In this setting, an algorithm sequentially queries points (or arms) and observes stochastic rewards, until it decides to stop and recommend a point whose expected reward is believed to be (close to) optimal. In particular, for $\eps$-best arm identification ($\eps$-BAI), algorithms are required upon stopping to recommend a point (or multiple points) whose expected reward is $\eps$-optimal with probability at least $1-\gamma$;\footnote{We write $\gamma$ for the risk level, since the letter $\delta$ will be used for another purpose.} see, e.g., \cite{evendar02-PACboundsMAB,mannor04-sampleComplexityExplorationMAB,evendar06-actionElimination,garivier21-nonAsymptoticSequentialTests}. This so-called $(\eps,\gamma)$-PAC condition is a statistical analog of getting a certificate $\xi_t \leq \eps$ in our setting. Furthermore, the main goal of $\eps$-BAI is to minimize the sample complexity (i.e., the (expected) total number of queries before stopping), which is analogous to minimizing the cost complexity in our multi-fidelity setting. As we will show in Section~\ref{sec:stochastic}, the connection between the two problems can be made explicit: our algorithm for certified multi-fidelity Lipschitz optimization can be used to solve an instance of $\eps$-BAI in Lipschitz bandits. Though this problem has been studied by \cite[Appendix~F]{wang21-fastPureExplorationFranckWolfe} for $\eps=0$ and a finite set of arms $\cX$, we are unaware of any earlier results for $\eps$-BAI in Lipschitz bandits with a continuous set $\cX \subset \R^d$. The special case of linear bandits has however received a lot of attention, either with finitely many arms (e.g., \cite{soare14-BAILinearBandits,degenne20-gamificationPureExplorationLinearBandits,jedra20-BAILinearBandits,kazerouni21-BAIGeneralizedLinearBandits,wang21-fastPureExplorationFranckWolfe,jourdan22a-choosingAnswersEpsBAILinearBandits}) or continuously many arms (e.g., \cite{jedra20-BAILinearBandits,bhat22-identifying_near_optimal_decis}, or \cite{carlsson24-pureExplorationLinearConstraints} for a related problem).

\paragraph{Adaptivity to smoothness versus error certificates}
Another important series of works is about adaptivity to the unknown smoothness of $f$, that is, the question of achieving nearly optimal optimization performances with an algorithm that has (almost) no prior knowledge on the smoothness of $f$. Among the many algorithms designed to that end, let us mention the seminal DIRECT algorithm of \cite{jones1993lipschitzian}, the $Z(k)$ algorithm of \cite{horn06-UnknownLipschitz}, as well as bandit algorithms (with simple or cumulative regret guarantees) including (Sto)SOO \cite{munos2011optimistic,valko2013stochastic}, the two-phase algorithm of \cite{bubeck11-withoutLipschitzConstant}, POO and GPO \cite{grill2015black,SKV19-GeneralParallelOptimizationWithoutMetric}, AdaLIPO \cite{malherbe2017global}, SequOOL and StroquOOL \cite{bartlett2019simple}, and their multi-fidelity variants \cite{sen2018multi,sen2019noisy,fiegel2020adaptive}. See also \cite{locatelli2018adaptivity} for a detailed account on possible and impossible adaptivity results in the single-fidelity setting.

Though adaptivity to unknown smoothness is a key robustness feature of optimization algorithms, we stress that it is in a way \emph{incompatible} with the certificate requirement. For instance, as noted by \cite{bachoc2021instance} in the single-fidelity setting, when optimizing a Lipschitz function~$f$ with unknown Lipschitz constant $\Lip(f)$, it is impossible to produce a finite certificate $\xi_t$ after any number $t$ of evaluations of $f$, since there could be an arbitrarily steep bump in a yet unobserved input region. More formally, if $f$ has a maximizer $x^\star$ within the interior of $\cX$, for small $\eps>0$, the lower bound of \cite[Theorem~2]{bachoc2021instance} scales at least as $\bigl(L/\Lip(f)\bigr)^d$ when $L \to +\infty$, which implies that the minimum number of evaluations that certified algorithms need for the function $f$ is arbitrarily large, if we require such algorithms to output valid certificates for all Lipschitz functions $g$ with arbitrarily large Lipschitz constants $\Lip(g)$. The same intuitive remark applies to our multi-fidelity setting, by the lower bound of Theorem~\ref{thm:lower_bound} in Section~\ref{sec:lower_bound}.

\subsection{Notation}
\label{sec:notations}

We collect below some notation that is used all throughout the paper, including the key quantity $S_{\beta, L}(f,\eps)$ defined in \eqref{eq:Sbeta} below.

\paragraph{Standard notation} $\N = \{0, 1, 2, \ldots \}$ denotes the set of natural numbers, and $\N^*=\N\backslash\{0\}$ denotes the set of positive natural numbers. For any $x \in \R^d$ and $r >0$, we write $B(x,r) = \{u \in \R^d: \norm{u-x} \leq r \}$ for the closed ball centered at $x$ with radius $r$.

\paragraph{Lipschitz functions, $\eps$-optimal points, layers}
Let $\mathcal{F}_L$ denote the set of $L$-Lipschitz functions from $\mathcal{X}$ to $\R$. Let also $\diam(\mathcal{X}) = \sup_{x,y \in \mathcal{X}} \norm{x-y}$ denote the diameter of $\mathcal{X}$. Since $f$ is $L$-Lipschitz, the largest possible optimization error $f(x^\star)-f(x^*_t)$ is bounded by $\eps_0 := L\cdot\diam(\mathcal{X})$. In the sequel, we will thus only consider values $\eps \in (0,\eps_0)$, and set $m_\eps \defeq \left\lceil\log_2(\eps_0/\eps) \right\rceil$.
For any $1 \leq k \leq m_\eps-1$ we define the intermediate target error $\eps_k := \eps_0 2^{-k}$; we also set $\eps_{m_\eps} := \eps$.

For any $0 \leq a < b$, we denote the set of $a$-optimal points by $\mathcal{X}_a := \{x \in\mathcal{X} \sep f(x^\star) - f(x) \leq a\}$, and we define the \emph{layer} $\mathcal{X}_{(a, b]} := \{ x \in \mathcal{X} \sep a< f(x^\star) -f(x) \leq b\}$, which is the set of $b$-optimal points that are not $a$-optimal.

\paragraph{Packing number}
For any $r>0$, the $r$-packing number $\mathcal{N}(\mathcal{X}', r)$ of a subset $\mathcal{X}' \subset \mathcal{X}$ is the largest number $n$ of $r$-separated points $x'_1, \ldots, x'_n \in \mathcal{X}'$, that is, such that $\norm{x'_i - x'_j}>r$ for all $i\neq j \leq n$. (By convention, $\mathcal{N}(\mathcal{X}', r) = 0$ if $\mathcal{X}'$ is empty. Note also that $\mathcal{N}(\mathcal{X}', r) < +\infty$ since $\mathcal{X}$ is compact.)

\paragraph{The complexity quantity $S_{\beta, L}(f, \eps)$}
For any $\beta>0$, we set
\begin{equation}
  \label{eq:Sbeta}
  S_{\beta, L}(f, \eps) \defeq \Neps c\left(\beta\eps\right) + \sum_{k=1}^{m_\eps} \Ni{k}c\left(\beta\eps_k\right) \;.
\end{equation}

This quantity is a multi-fidelity generalization of the quantity $S_{C}(f,\eps)$ introduced by \cite{bouttier2020regret} in the single-fidelity setting. In Appendix~\ref{sec:propertiesExamples} we comment on the dependence on $\beta,L,\Vert\cdot\Vert$, and discuss two simple examples that will prove useful in interpreting our upper and lower bounds.

As we will see later, $S_{\beta, L}(f, \eps)$ plays a key role in the optimal cost complexity of certified algorithms. We briefly explain why.
Let $x \in \mathcal{X}_{(\eps_k, \eps_{k-1})}$. To realize that $x$ belongs to the layer $\mathcal{X}_{(\eps_k, \eps_{k-1})}$, a single call to $f$ at $x$ with (prophetic) evaluation accuracy  $\alpha \approx \eps_k$ would be enough, yielding a cost roughly of $c(\eps_k)$. By $L$-Lipschitz continuity of $f$, this single evaluation also helps classify nearby points (at distance roughly $\eps_k/L)$ within the same layer. Therefore, the whole layer could (hopefully) be identified with roughly $\Ni{k}$ evaluations of $f$ at accuracy $\alpha \approx \eps_k$. Repeating these arguments over all layers and using another (similar) argument over $\cX_\eps$ suggests that the sum $S_{\beta, L}(f, \eps)$ above characterizes the optimal cost complexity of certified algorithms. In the next sections we prove upper and lower bounds supporting this intuition.

\section{The Certified MFDOO Algorithm, and its Cost Complexity}
\label{sec:mfdoo}
In this section we define a certified version of the MFDOO algorithm \cite{sen2018multi}, and then study its cost complexity (Theorem~\ref{thm:upper_bound} below), which we will prove to be nearly optimal in Section~\ref{sec:lower_bound}.

Similarly to MFDOO \cite{sen2018multi} and its ancestors (e.g., the branch-and-bound algorithm of \cite{perevozchikov1990complexity}, HOO \cite{bubeck2011x}, DOO \cite{munos2011optimistic}, POO \cite{grill2015black} c.DOO \cite{bachoc2021instance}, etc), our algorithm takes as input a hierarchical partitioning of $\mathcal{X}$, that is, a tree-based structure $X$ in which each node represents a region of $\mathcal{X}$ and has $K$ children, which correspond to a $K$-partition of the parent region.
More precisely, $X$ is an infinite sequence of subsets $(X_{h,i})_{h \in \N, i \in \{0, \ldots, K^h-1\}}$ of $\mathcal{X}$ called \emph{cells} such that $\mathcal{X} \subset X_{0, 0}$, and for any depth $h \in \N$ and location $0\leq i \leq K^h-1$, the cells $X_{h+1, Ki}, \ldots, X_{h+1, K(i+1)-1}$ form a partition of $X_{h,i}$ (the nodes $(h+1, Ki), \ldots, (h+1, K(i+1)-1)$ are the \emph{children} of node $(h,i)$).
Each cell has a \emph{representative} $x_{h,i} \in X_{h,i}$. We assume that $x_{h,i} \in \mathcal{X}$ whenever $X_{h,i} \cap \mathcal{X} \neq \varnothing$. (A typical example is the barycenter of the cell, if inside $\mathcal{X}$.)

In all the paper, we make the following two assumptions. The first one is classical (e.g., \cite{bubeck2011x,munos2011optimistic}). The second one appeared in \cite{bachoc2021instance} and was useful to derive bounds on DOO or c.DOO in terms of packing numbers of layers $\mathcal{X}_{(\eps_k,\eps_{k-1}]}$. Both assumptions can always be satisfied when $\mathcal{X}$ is compact. For example, if $\mathcal{X} = [0,1]^d$ and $\Vert \cdot \Vert$ is the sup norm, we can take the regular dyadic partitioning $(X_{h,i})_{h \in \N, i \in \{0, \ldots, 2^{d h}-1\}}$ consisting of $2^{d h}$ cubes of size $2^{-h}$ at depth $h \geq 0$, with centers $x_{h,i}$. Then, Assumptions~\ref{assum:diameter} and~\ref{assum:nu} hold true with $R=1$ and $\delta=\nu=1/2$. If $\cX \subset \R^d$ is any other compact set, one way to get valid values for $R$ and $\nu$ with $\delta=1/2$ consists in considering the smallest hypercube containing $\cX$. Note that other values of $\delta,R,\nu$ might be possible, but this does not affect the rate of our cost complexity bounds as $\eps \to 0$, under a mild assumption on the cost function $c$ (see Appendix~\ref{sec:propertiesExamples}).

\begin{assumption}
  \label{assum:diameter}
  There exist two positive constants $\delta \in (0, 1)$ and $R>0$ such that, for all $h \in \N$, $i \in \{0,\ldots,K^h-1\}$, and all $u, v \in X_{h,i}$, we have $\norm{u-v} \leq R\delta^h$.
\end{assumption}

\begin{assumption}
  \label{assum:nu}
  There exists $\nu>0$ such that, with $\delta$ as in Assumption \ref{assum:diameter}, for any $h,h' \in \N$, $i \in \{0, \ldots, K^h-1\}$ and $i' \in \{0, \ldots, K^{h'}-1\}$ with $(h,i) \neq (h', i')$, the representatives $x_{h,i}$ and $x_{h',i'}$ defined above satisfy $\norm{x_{h,i} - x_{h',i'}} \geq \nu \delta^{\max\{h, h'\}}$.
\end{assumption}

We now define \mfdoo{} (Certified Multi-Fidelity Deterministic Optimistic Optimization). The pseudo-code is given in Algorithm \ref{algo:mfdoo} below. The algorithm maintains a set $\mathcal{L}_t$ of active nodes (or \emph{leaves}) whose associated cells cover $\mathcal{X}$.
At the end of each iteration (Line \ref{eq:argmax}), \mfdoo{} picks the most promising leaf $(h^*, i^*)$ by maximizing the surrogate $y_{h,i} + LR\delta^h + \alpha_{h,i}$, which is an upper bound on $f(x)$ for any $x \in X_{h,i}$.\footnote{Indeed, $y_{h,i}$ is an $\alpha_{h,i}$-approximation of $f(x_{h,i})$, $f$ is $L$-Lipschitz, and the maximum distance from a point in $X_{h,i}$ to $x_{h,i}$ is at most $R\delta^h$, so that $\max_{x \in X_{h,i}} f(x) \leq f(x_{h,i}) + LR\delta^h\leq y_{h,i} + LR\delta^h + \alpha_{h,i}$.}
This step is an instance of the so-called \emph{optimism principle} (or \emph{optimism in the face of uncertainty}), which was used multiple times in the past, in the stochastic or deterministic zeroth-order (bandit)  optimization literatures (e.g., \cite{AuCeFi-02-FiniteTimeBandits} for the UCB1 algorithm, as well as earlier or later references that can be found, e.g., in \cite{perevozchikov1990complexity,munos2011optimistic,munos2014bandits,kleinberg2019bandits,LS20-banditalgos,slivkins2019introduction}).

Then, during the next iteration, \mfdoo{} develops the tree by querying one after the other all the (feasible) children of $(h^*, i^*)$.
At Line~\ref{eq:pick}, it picks $x_t = x_{h^*+1,j}$ as the next query point and $\alpha_t = LR \delta^{h^*+1}$ for the accuracy. (A much smaller value of $\alpha_t$ could be counter-productive: it could come at a much higher cost, while not improving the optimization process by much, since the surrogate is an over-approximation of $f$ with a mistake possibly of the order of $LR\delta^{h^*+1}$ on the cell $X_{h^*+1,j}$ even if $f$ were observed exactly.)
After receiving the approximate evaluation $y_t$ of $f(x_t)$, \mfdoo{} returns the recommendation $x_t^* = x_{\tilde{t}}$ (see Line \ref{eq:recommend}), which is the point with the currently best guaranteed value of~$f$. Finally, the certificate $\xi_t$ at Lines~\ref{eq:certificate} or~\ref{eq:updatecertificate} is the difference between a guaranteed upper bound $y_{h^*, i^*} + LR\delta^{h^*} + \alpha_{h^*,i^*}$ on $\max(f)$ and a guaranteed lower bound $y_{\tilde{t}} -\alpha_{\tilde{t}}$ on $f(x_t^*)$.

Note that in Algorithm~\ref{algo:mfdoo} and in the rest of the paper, we identify any round $t \geq 1$ with the node $(h,i)$ that is queried at time~$t$.\footnote{\label{ft:injection}By definition of Algorithm~\ref{algo:mfdoo}, there is indeed an injection $t \in \N^* \mapsto (h,i)$ with $h \in \N, i \in \{0, \ldots, K^h-1\}$.} Depending on our needs, we index quantities either by rounds or nodes (writing, e.g., $y_t$ or $y_{h,i}$).

\begin{algorithm}
  \caption{\mfdoo{} (Certified Multi-Fidelity Deterministic Optimistic Optimization)}
  \label{algo:mfdoo}
  \hspace*{\algorithmicindent} \textbf{Inputs}: $\mathcal{X}$, $K$, $(X_{h,i})_{h \in \N, i \in \{0,\ldots, K^h-1\}}$, $(x_{h,i})_{h \in \N, i \in \{1,\ldots, K^h-1\}}$, $\delta$, $R$, $L$\\
  \hspace*{\algorithmicindent} \textbf{Initialization} Let $t\leftarrow 1$ and $\mathcal{L}_1 \leftarrow \{(0,0)\}$
  \begin{algorithmic}[1]
    \STATE Pick the first query point $x_1 \leftarrow x_{0,0}$, and the first accuracy $\alpha_1 \leftarrow LR$
    \STATE Observe the value $y_1 = E_1(x_1, \alpha_1) \in [f(x_1) - \alpha_1, f(x_1)+\alpha_1]$
    \STATE \label{eq:certificateInit} Output the recommendation $x_1^* \leftarrow x_1$ and certificate $\xi_1 \leftarrow LR$
    \STATE \label{eq:rootnode} Pick the first node $(h^*, i^*) \leftarrow (0,0)$
    \FOR{$iteration=1, 2, \ldots$}{ \label{eq:iteration}
      \FORALL{child $(h^*+1, j)$ of $(h^*,i^*)$}{
          \IF{$X_{h^*+1,j} \cap \mathcal{X} \neq \varnothing$}
          \STATE \label{eq:insertleaf} Let $t \leftarrow t+1$ and $\mathcal{L}_t \leftarrow \mathcal{L}_{t-1} \cup \{(h^*+1,j)\}$
          \STATE \label{eq:pick} Pick the query point $x_t \leftarrow x_{h^*+1,j}$ and accuracy $\alpha_t \leftarrow LR\delta^{h^*+1}$
          \STATE Observe the value $y_t = E_t(x_t, \alpha_t) \in [f(x_t) - \alpha_t, f(x_t) + \alpha_t]$ given by $E$
          \STATE \label{eq:recommend} Output the recommendation $x_t^* = x_{\tilde{t}},$ with $\tilde{t} \in \argmax_{1 \leq s \leq t} \{y_s - \alpha_s\}$
          \STATE \label{eq:certificate} Output the certificate $\xi_t = y_{h^*, i^*} + LR\delta^{h^*} + \alpha_{h^*,i^*} - (y_{\tilde{t}} -\alpha_{\tilde{t}})$
          \ENDIF
      }
      \ENDFOR
    }
    \STATE \label{eq:removenode} Remove $(h^*, i^*)$ from $\mathcal{L}_t$
    \STATE \label{eq:argmax} Let $(h^*, i^*) \in  \argmax_{(h, i) \in \mathcal{L}_t} \{y_{h,i} + LR\delta^h + \alpha_{h,i}\}$
    \STATE \label{eq:updatecertificate} Update the last certificate $\xi_t = y_{h^*, i^*} + LR\delta^{h^*} + \alpha_{h^*,i^*} - (y_{\tilde{t}} -\alpha_{\tilde{t}})$
    \ENDFOR
  \end{algorithmic}
\end{algorithm}

\paragraph{Time and space complexities.}
Similarly to earlier algorithms of this type (e.g., DOO \cite{munos2011optimistic}, c.DOO \cite{bachoc2021instance} and MFDOO \cite{sen2018multi}), \mfdoo{} is computationally tractable, if we ignore the cost of evaluating $x \mapsto f(x)$.\footnote{To be rigorous, we also ignore the cost of evaluating $(h,i) \mapsto x_{h,i}$. This can be done in constant time when a closed-form or recursive formula (using the previously computed parent node's representative) is available.} Indeed, using a binary max-heap, all the operations at Lines~\ref{eq:insertleaf}, \ref{eq:removenode} and \ref{eq:argmax} can be executed in (at most) logarithmic time in the number $|\mathcal{L}_t|$ of active nodes, which is $\mathcal{O}(t)$. All the other steps, including computing $\tilde{t}$ in Line~\ref{eq:recommend} (in a sequential fashion), can be executed in constant time at every round. Therefore, the total running time of \mfdoo{} is $\mathcal{O}(t \ln t)$ after $t$ rounds. Likewise, the memory footprint can be seen to be $\mathcal{O}(t)$ up to round $t$.

\medskip
We now analyze the behavior of \mfdoo{}. Before bounding its cost complexity,  we start by proving that the $\xi_t$'s defined at Lines~\ref{eq:certificateInit}, \ref{eq:certificate} and \ref{eq:updatecertificate} are valid certificates.

\begin{lemma}
  \label{lemma:certification}
  Suppose Assumption \ref{assum:diameter} holds, and that $f:\mathcal{X} \to \R$ is an $L$-Lipschitz function, with a maximizer $x^\star\in \mathcal{X}$.
  Then, for any environment $E \in \mathcal{E}(f)$ and any $t \in \N^*$, the quantity $\xi_t$  defined at Lines~\ref{eq:certificateInit}, \ref{eq:certificate} and \ref{eq:updatecertificate} of Algorithm \ref{algo:mfdoo} is a valid certificate, that is: $f(x^\star) - f(x^*_t) \leq \xi_t$.
\end{lemma}

\begin{proof}
  Since $f$ is $L$-Lipschitz and $R \geq \diam(\mathcal{X})$, note that $\xi_1 = LR \geq f(x^\star) - f(x_1^*)$. We now prove the lemma for any subsequent round. For any $t' \geq 2$, consider the moment when the algorithm reaches Line~\ref{eq:certificate} with $t=t'$. Next we show that the certificate $\xi_t =y_{h^*, i^*} + LR\delta^{h^*} + \alpha_{h^*, i^*} - \max_{s\leq t} (y_s - \alpha_s)$ defined at that time satisfies $\xi_{t} \geq f(x^\star) - f(x_{t}^*)$ (the potential update at Line~\ref{eq:updatecertificate} will be addressed at the end of the proof). The associated node $(h^*, i^*)$ was either defined at Line~\ref{eq:rootnode} if the outer \texttt{for} loop is still at iteration $1$, or at Line~\ref{eq:argmax} otherwise. In both cases,  we have $(h^*, i^*) \in  \argmax_{(h, i) \in \mathcal{L}_m} \{y_{h,i} + LR\delta^h + \alpha_{h,i}\}$ for a number $m \leq t-1$ of evaluations of $f$.

  Note that, by induction on the \texttt{iteration} variable, the cells $X_{h,i}$ associated with the leaves $(h,i) \in \mathcal{L}_m$ form a partition of a superset of $\mathcal{X}$. Let $(\bar{h}, \bar{i}) \in \mathcal{L}_m$ be the node of the cell $X_{\bar{h}, \bar{i}}$ containing $x^\star$.
  By the maximizing property of $(h^*,i^*)$, we have
  \begin{align}
    \nonumber y_{h^*, i^*} + LR\delta^{h^*} +\alpha_{h^*, i^*} & \geq y_{\bar{h}, \bar{i}} + LR\delta^{\bar{h}} +\alpha_{\bar{h}, \bar{i}}\\
    & \geq f(x_{\bar{h}, \bar{i}}) + LR \delta ^{\bar{h}} \geq f(x^\star) \;,\label{eq:is_certificate}
  \end{align}
  where the last line follows from $|y_{\bar{h}, \bar{i}} - f(x_{\bar{h}, \bar{i}})| \leq \alpha_{\bar{h}, \bar{i}}$, Assumption~\ref{assum:diameter}, and the fact that $f$ is $L$-Lipschitz. To conclude, note that  $y_s - \alpha_s \leq f(x_s)$ for all $s \leq t$, and therefore $y_{\tilde{t}} - \alpha_{\tilde{t}} \leq f(x_{\tilde{t}}) = f(x^*_t)$. Combining this with \eqref{eq:is_certificate} entails that $\xi_t \geq f(x^\star) - f(x_t^*)$. Noting that the same arguments apply (with $m=t$) if $\xi_t$ is redefined at Line~\ref{eq:updatecertificate} concludes the proof.
\end{proof}

We just proved that \mfdoo{} is a certified algorithm. We now show that its cost complexity can be controlled in terms of the quantity $S_{\beta, L}(f, \eps)$ defined in \eqref{eq:Sbeta}. We recall that $\eps_0 \defeq L \, \diam(\mathcal{X})$, $m_\eps \defeq \left\lceil\log_2(\eps_0/\eps) \right\rceil$, $\eps_{m_\eps} \defeq \eps$, and $\eps_k \defeq \eps_0 2^{-k}$ for all $1 \leq k \leq m_\eps-1$.

\begin{theorem}
  \label{thm:upper_bound}
  Assume that $c: \R^+ \to \R$ is non-increasing and that $\mathcal{X} \subset \R^d$ is compact. Suppose that Assumptions~\ref{assum:diameter} and \ref{assum:nu} hold, and that for some known $L>0$, \mfdoo{} is run with evaluation accuracies $\alpha_{h,i} = LR\delta^h$ (cf. Algorithm \ref{algo:mfdoo}). Then there exists a constant $a>0$ (e.g., $a = K$ if $\nu \geq 3 R$ or $a= K (1+6 R/\nu)^d$ otherwise) such that, for any $L$-Lipschitz function $f:\mathcal{X} \to \R$, any environment $E \in \mathcal{E}(f)$, and any $\eps \in (0,\eps_0)$,
  \[
  \sprecAE{\mfdoo{}}{E} \leq a S_{\frac{\delta}{3}, L}(f, \eps) + c(LR) \;.
  \]
\end{theorem}

Recall from Section~\ref{sec:mfdoo} that $\delta$, $R$ and $\nu$ are constants that do not depend on $f$ nor $\varepsilon$. In particular, the bound in Theorem~\ref{thm:upper_bound} is of the order of $\left(\frac{L}{\eps}\right)^d c(\delta \eps/3)$ when $f$ is a constant function, and of the order of $c(\delta \eps/3)$ (up to a log factor) when $f(x) = 1-|x|$ for some norm $|\cdot|$ on $\R^d$. See Appendix~\ref{sec:propertiesExamples} for details. \\

We make additional comments before proving the theorem.

\paragraph{Related upper bounds}
Similar upper bounds were proved in \cite[Theorem~3.6]{bouttier2020regret} and \cite[Theorem~1]{bachoc2021instance} when $f$ is evaluated perfectly, which corresponds to the special case $c(\alpha)=1$ for all $\alpha$. Theorem~\ref{thm:upper_bound} above generalizes these results (up to constants) to the multi-fidelity setting.

Other related results are bounds for MFDOO \cite{sen2018multi} and Kometo \cite{fiegel2020adaptive}, which are multi-fidelity algorithms \emph{without certificates}. In that setting, the performances are measured differently. The cost complexity can be defined as the total cost incurred by the algorithm before outputting an $\eps$-optimal recommendation (the difference with the certified setting is that the learner has no \emph{observable} proof that an $\eps$-maximizer has been found.)
With such performance measure, MFDOO satisfies a complexity bound similar to $S_{\beta, L}(f, \eps)$ but \emph{without the first term $\Neps c\left(\beta\eps\right)$} in \eqref{eq:Sbeta}.\footnote{This bound can be proved along the same lines as those of Theorem~\ref{thm:upper_bound}. See also \cite[Theorem~1]{sen2018multi} and \cite[Theorem~3]{fiegel2020adaptive} for similar bounds under slightly weaker assumptions (relating $f$ directly to the hierarchical partitioning) but that are expressed in terms of a near-optimality dimension of $f$, and thus do not reflect the fact that constant functions are easy to optimize. Roughly speaking, the bound of \cite[Theorem~1]{sen2018multi} for MFDOO is in spirit close to $\sum_{k=1}^{m_\eps} \mathcal{N}\bigl(\mathcal{X}_{\eps_{k-1}},\eps_k/L\bigr) c(\beta\eps_k)$, instead of the tighter bound $\sum_{k=1}^{m_\eps} \mathcal{N}\bigl(\mathcal{X}_{(\eps_k,\eps_{k-1}]},\eps_k/L\bigr) c(\beta\eps_k)$.}
This difference can be negligible for some functions (e.g., if $c_1 \norm{x-x^\star}^\nu \leq f(x^\star)-f(x) \leq c_2 \norm{x-x^\star}^\nu$ for all $x \in \mathcal{X}$ and some $c_1,c_2>0$ and $\nu \geq 1$, where $x^\star$ is a maximizer of $f$) and under a mild condition on $c$, but it can be dramatic for other functions. For instance, for constant functions, the term $\Neps c\left(\beta\eps\right)$ is of the order of $(L/\eps)^d c(\beta \eps)$. The reason behind this large additional term in the certified setting is intuitive: a constant function $f$ is perfectly optimized after one evaluation only, but \emph{certifying} the result at accuracy~$\eps$ somehow requires to evaluate the function on a $\eps/L$-cover of $\mathcal{X}$ with accuracies $\alpha_t \approx \eps$, so as to make sure no bumps of size~$\eps$ were forgotten.

Note also that, contrary to \cite{fiegel2020adaptive}, we work with a \emph{known} bias function but an \emph{unknown} cost function. This is because we aim at \emph{certifying} an $\eps$-maximizer of $f$, rather than optimally allocating a total evaluation budget $\Lambda$.

Finally, note that the bound of Theorem~\ref{thm:upper_bound} can have a much worse dependency in $\eps$ than what could be obtained under strong structural assumptions on $f$. For example, if $f$ is smooth and strongly concave (which corresponds to $\nu=2$ in the paragraph before last) and can be evaluated perfectly ($c(\alpha)=1$ for all $\alpha$), the bound of Theorem~\ref{thm:upper_bound} can be of the order of $(1/\eps)^{d/2}$. On the other hand, as recalled in \cite[Section~4]{larson19-derivativeFreeOptimizationMethods}, some zeroth-order algorithms achieve a sample complexity of $\mathcal{O}\bigl(\ln(1/\eps)\bigr)$ if $\cX=\R^d$. Beyond the difference between constrained and unconstrained optimization, a key reason for our worse bound is that we require the algorithm to output valid certificates for all $L$-Lipschitz functions (a much larger function class). This is in the same spirit as the remark about the impossible adaptivity to smoothness in Section~\ref{sec:relatedworks}. However it is likely that we could tailor \mfdoo{} to smooth and strongly concave functions by using a tweaked exploration bonus, as was done for DOO with semi-metrics \cite{munos2011optimistic} in the single-fidelity setting. We leave this interesting question for future work.

\paragraph{On the choice of $\eps_k$} As can be seen from the proof below, the upper bound is actually true for any decreasing sequence $\eps_0 = L \, \diam(\mathcal{X}) > \eps_1 > \ldots > \eps_{m-1} > \eps_m = \eps$ and any $m \geq 1$. The specific sequence $\eps_k \defeq \eps_0 2^{-k}$ however realizes a good trade-off between small ratios $\eps_{k-1}/\eps_k \leq 2$ and a small number of terms $m_{\eps} = \left\lceil\log_2(\eps_0/\eps) \right\rceil$. The nearly-matching lower bound of Section~\ref{sec:lower_bound} will indeed imply that this sequence is nearly optimal.

\paragraph{Possible improvements or consequences} Note that the constant $a$ was not optimized and could likely be improved. Furthermore, similarly to  \cite{bachoc2021instance}, under a mild geometric condition on~$\mathcal{X}$ recalled in Section~\ref{sec:stochastic}, the sum $S_{\delta/3, L}(f, \eps)$ can be bounded (up to multiplicative constants) in between two integrals of the form $\int_{\mathcal{X}} c\bigl(b \cdot (\Delta_x+\eps)\bigr)/(\Delta_x+\eps)^d d x$, where $\Delta_x = f(x^\star)-f(x)$ and $b \in \bigl\{\delta/3,\delta/12\bigr\}$ (provided $\eps < \eps_0/2$). This integral form is omitted due to lack of space.

\begin{proof}[Proof of Theorem~\ref{thm:upper_bound}]
  The proof generalizes that of \cite[Theorem~1]{bachoc2021instance} to the multi-fidelity setting, with similar arguments yet a few technical subtleties.   
  In order to bound the total cost incurred by \mfdoo{} against environment~$E$, we control the index $I_\eps \geq 1$ of the first iteration (cf Line \ref{eq:iteration}) at the end of which the certificate falls below~$\eps$. More precisely, let $(h^*_\ell, i^*_\ell)$ be the node chosen at the end of each iteration $\ell \geq 1$ (Line~\ref{eq:argmax}). Then, we define $I_\eps$ by\footnote{The rest of the proof implies that the set is never empty, so that $I_\eps < +\infty$.}
  \[
  I_\eps = \inf \Bigl\{ \ell \in \N^* \sep y_{h^*_\ell, i^*_\ell} + LR\delta^{h^*_\ell} + \alpha_{h^*_\ell, i^*_\ell} \leq \max_{s \leq T_\ell} \{y_s - \alpha_s\} + \eps \Bigr\} \;,
  \]
    where for any $\ell \geq 1$, the quantity $T_\ell$ denotes the total number of evaluations of $f$ until the leaf $(h^*_\ell,i^*_\ell)$ is selected at Line \ref{eq:argmax}. Next we focus on $\tau \defeq T_{I_{\eps}}$. Note that $\xi_\tau \leq \eps$ by definition of $I_\eps$ and $\xi_\tau$ (in Line~\ref{eq:updatecertificate}). Recalling that $\sprecAE{\mfdoo{}}{E}$ is the total cost that \mfdoo{} incurs until outputting a certificate below $\eps$ for the first time, this entails
  \begin{equation}
    \sprecAE{\mfdoo{}}{E} \leq \sum_{t=1}^\tau c(\alpha_t)\;.    
    \label{eq:upperbound-sigma}
  \end{equation}
  We now split the right-hand side into several terms involving the layers $\mathcal{X}_{(\eps_k, \eps_{k-1}]}$. We set $(h^*_0, i^*_0) = (0,0)$. Note that the points $x_t$ queried at times $t \in \{2, \ldots, \tau\}$ are all associated with nodes $(h,i)$ that are children of some $(h^*_\ell, i^*_\ell)$, $\ell=0, \ldots, I_\eps-1$, and that these $(h,i)$ are queried only once. Therefore,
  \begin{align}
    \sum_{t=1}^\tau c(\alpha_t) 
    &\leq c(\alpha_1) +  \sum_{\ell=0}^{I_\eps-1} \; \sum_{j=K i^*_\ell}^{K (i^*_\ell+1)-1} c\bigl(\alpha_{h^*_\ell+1,j}\bigr) = c(LR) + K \sum_{x_{h^*, i^*} \in \mathcal{E}_\eps} c\bigl(LR\delta^{h^*+1}\bigr) \;,
      \label{eq:bound_first_step}
  \end{align}
  where we set $\mathcal{E}_\eps \defeq \{x_{h_0^*, i_0^*}, \ldots, x_{h_{I_\eps-1}^*, i_{I_\eps-1}^*}\}$ (the $x_{h, i}$ are pairwise-distinct by Assumption~\ref{assum:nu}).
  
  We now split the sum over $\mathcal{E}_\eps$ above into $m_\eps + 1 = \lceil \log_2(\eps_0/\eps) \rceil+1$ terms.
  Recall from Section~\ref{sec:notations} that $\eps_0 = L \, \diam(\mathcal{X})$, $\eps_k = \eps_0 2^{-k}$ for $1\leq k \leq m_\eps-1$,  and $\eps_{m_\eps} = \eps$. Since the sets $\mathcal{X}_\eps$ (all $\eps$-optimal points) and $\mathcal{X}_{(\eps_k, \eps_{k-1}]}$, $k=1,\ldots,m_\eps$ (all points in between $\eps_k$ and $\eps_{k-1}$ optimal) form a partition of $\mathcal{X}$,
  \begin{equation}
    \label{eq:X_in_Xepsk}
    \mathcal{E}_\eps = \left(\mathcal{E}_\eps \cap \mathcal{X}_\eps \right) \cup \bigcup_{k=1}^{m_\eps} \left(\mathcal{E}_\eps \cap \mathcal{X}_{(\eps_k, \eps_{k-1}]}\right).
    \end{equation}
  Let $N_{\eps, k}$ be the cardinality of $\mathcal{E}_\eps \cap \mathcal{X}_{(\eps_k, \eps_{k-1}]}$ for all $1 \leq k \leq m_\eps$ and $N_{\eps, m_\eps+1}$ be the cardinality of $\mathcal{E}_\eps \cap \mathcal{X}_\eps$.
  Moreover, let $h_{\eps, k}$ be the maximum depth $h^*$ reached by points $x_{h^*,i^*}$ in $\mathcal{E}_\eps \cap \mathcal{X}_{(\eps_k, \eps_{k-1}]}$ for $1 \leq k \leq m_\eps$, and $h_{\eps, m_\eps+1}$ be the maximum depth reached by points in $\mathcal{E}_\eps \cap \mathcal{X}_\eps$.
  By \eqref{eq:bound_first_step}, \eqref{eq:X_in_Xepsk}, and the fact that $\alpha \mapsto c(\alpha)$ is non-increasing, we have:
  \begin{equation}
    \sum_{t=1}^\tau c(\alpha_t)\leq c(LR) + K\sum_{k=1}^{m_\eps+1} N_{\eps, k} \, c\bigl(LR\delta^{h_{\eps,k}+1}\bigr) \;.
    \label{eq:bound_third_step}
  \end{equation}
  We now bound $N_{\eps,k}$ from above and $LR\delta^{h_{\eps,k}+1}$ from below (see \eqref{eq:epsk_delta}, \eqref{eq:NpositiveLayers}, \eqref{eq:eps_delta}, and \eqref{eq:Nlayer0}).
  
  We start by proving \eqref{eq:XLR} below. Let $x^\star \in \mathcal{X}$ be a maximizer of $f$.
  Following the same arguments as before \eqref{eq:is_certificate} (using $f$ being $L$-Lipschitz and $E \in \mathcal{E}(f)$), we can see that, for any node $(h^*, i^*)$ selected at Line~\ref{eq:argmax},
  \[
  y_{h^*, i^*} + LR\delta^{h^*} + \alpha_{h^*, i^*} \geq f(x^\star) \;.
  \]
  This implies that $f(x_{h^*, i^*}) + LR\delta^{h^*} + 2\alpha_{h^*, i^*} \geq f(x^\star)$, and thus $3 L R\delta^{h^*} \geq f(x^\star) - f(x_{h^*, i^*})$ (by $\alpha_{h^*,i^*} = LR\delta^{h^*}$). Therefore, for any $\ell \in \{0,\ldots,I_\eps-1\}$ (the case $\ell=0$ is straightforward),
  \begin{equation}
    \label{eq:XLR}
    x_{h^*_\ell, i^*_\ell} \in \mathcal{X}_{3 L R \delta^{h^*_\ell}}\;.
  \end{equation}

  Now, let $k \in \{1,\ldots,m_\eps\}$ and $x_{h^*_\ell, i^*_\ell} \in \mathcal{X}_{(\eps_k, \eps_{k-1}]} \cap \mathcal{E}_\eps$. By \eqref{eq:XLR} and the fact that $x_{h^*_\ell, i^*_\ell}  \in \mathcal{X}_{(\eps_k, \eps_{k-1}]}$ is not $\eps_k$-optimal, we have $3 L R \delta^{h^*_\ell} > \eps_k$. This and the definition of $h_{\eps,k}$ entail
  \begin{equation}
    3 L R \delta^{h_{\eps,k}} > \eps_k \;.
    \label{eq:epsk_delta}
  \end{equation}
  
  Also, let $x_{h,j}$ and $x_{h', j'}$ be two distinct elements of $\mathcal{X}_{(\eps_k, \eps_{k-1}]} \cap \mathcal{E}_\eps$. By Assumption~\ref{assum:nu} and \eqref{eq:epsk_delta}, we have $\norm{x_{h,j} -x_{h',j'}} \geq \nu \delta^{\max\{h, h'\}} > \frac{\nu \eps_k}{3 LR}$.
  Therefore, and by definition of a packing number, we get that for all $k \in \{1,\ldots,m_\eps\}$, the cardinality $N_{\eps,k}$ of $\mathcal{E}_\eps \cap \mathcal{X}_{(\eps_k, \eps_{k-1}]}$ satisfies
  \begin{align}
    N_{\eps, k}
    &\leq \mathcal{N}\left(\mathcal{X}_{(\eps_{k}, \eps_{k-1}]}, \frac{\nu\eps_{k}}{3LR}\right) \leq  \underbrace{\left(\mathbbm{1}_{\frac{\nu}{3R} \geq 1} + \mathbbm{1}_{\frac{\nu}{3R} < 1}\left(1+\frac{6R}{\nu}\right)^d\right)}_{=:b} \; \Ni{k} \label{eq:NpositiveLayers}
  \end{align}
  from Lemma \ref{lemma:packing_number} in Appendix \ref{sec:proofs}.

  Now, let $x_{h^*_\ell, i^*_\ell} \in \mathcal{X}_\eps \cap \mathcal{E}_\eps$, with $\ell \in \{0,\ldots,I_\eps-1\}$. If $\ell \geq 1$, we have, by definition of $I_\eps$,
  \begin{align*}
    y_{h^*_\ell, i^*_\ell} + LR\delta^{h^*_\ell} + \alpha_{h^*_\ell, i^*_\ell} & > \max_{s \leq T_\ell} \{y_s - \alpha_s\} + \eps \geq y_{h^*_\ell, i^*_\ell} - \alpha_{h^*_\ell, i^*_\ell} +\eps\;.
  \end{align*}
  Again, replacing $\alpha_{h^*_\ell, i^*_\ell}$ with $LR\delta^{h^*_\ell}$, we get $3LR\delta^{h^*_\ell} > \eps$, which is also true if $\ell=0$. Therefore,
  \begin{equation}
    \label{eq:eps_delta}
    3LR\delta^{h_{\eps,m_\eps+1}} > \eps\;.
  \end{equation}
  Combining this inequality with Assumption \ref{assum:nu}, we get that $\norm{x_{h,j} -x_{h',j'}} > \frac{\nu \eps}{3 LR}$ for any two distinct elements $x_{h,j}, x_{h', j'}$ of $\mathcal{X}_{\eps} \cap \mathcal{E}_\eps$.
  Therefore, and by definition of a packing number,
  \begin{equation}
    N_{\eps, m_\eps+1}
    \leq \mathcal{N}\left(\mathcal{X}_{\eps}, \frac{\nu\eps}{3LR}\right)
    \leq b \cdot  \Neps \;,
    \label{eq:Nlayer0}
  \end{equation}
  by Lemma \ref{lemma:packing_number} again.
  Putting \eqref{eq:upperbound-sigma}, \eqref{eq:bound_third_step}, \eqref{eq:NpositiveLayers}, \eqref{eq:Nlayer0} together and setting $a \defeq K b$, we get
  \begin{align*}
    & \sprecAE{\mfdoo{}}{E} \\
    & \qquad \leq c(LR) + a \Neps c\left(LR\delta^{h_{\eps,m_\eps+1}+1}\right) + a \sum_{k=1}^{m_\eps} \Ni{k} c\left(LR\delta^{h_{\eps,k}+1}\right) \\
    & \qquad \leq c(LR) + a\Neps c\left(\frac{\delta\eps}{3}\right) + a\sum_{k=1}^{m_\eps} \Ni{k}c\left(\frac{\delta\eps_k}{3}\right) \;,
  \end{align*}
  where we used \eqref{eq:epsk_delta}, \eqref{eq:eps_delta}, and the fact that $c$ is non-increasing. This concludes the proof.
\end{proof}

\section{Lower Bound}
\label{sec:lower_bound}

In this section, for any fixed $L$-Lipschitz function $f:\mathcal{X} \to \R$, we derive a lower bound on the worst-environment cost complexity $\sup_{E \in \mathcal{E}(f)} \sprec$ (see \eqref{eq:defcostcomplexity}) of any certified algorithm $A$. Our main result below, which depends on $f$ through the key quantity $S_{\beta, L}(f, \eps)$ defined in \eqref{eq:Sbeta}, generalizes \cite[Theorem~2]{bachoc2021instance} from perfect evaluations of $f$ to the multi-fidelity setting. We recall that $\eps_0 = L \, \diam(\mathcal{X})$ and $m_\eps = \left\lceil\log_2(\eps_0/\eps) \right\rceil$.

\begin{theorem}
  \label{thm:lower_bound}
  Assume that $c: \R_+ \to \R_+$ is a non-increasing function and that $\mathcal{X} \subset \R^d$ is a compact and connected set. Then, for some constant $a_d>0$ (e.g., $a_d = 1/ 65^d$), the cost complexity of any certified algorithm $A$ satisfies, for any $L$-Lipschitz function $f\colon \mathcal{X}\to \R$ and any target optimization error $\eps \in (0, \eps_0/2)$,
  \[
  \sup_{E \in \mathcal{E}(f)} \sprec \geq \frac{a_d \bigl(1-\Lip(f)/L\bigr)^d}{1+m_\eps} \; S_{16, L}(f, \eps)\;.
  \]
  Importantly, the lower bound holds for \emph{certified} algorithms, which by definition are required to output valid certificates for any $L$-Lipschitz function $f\colon \mathcal{X}\to \R$ and any environment $E \in \mathcal{E}(f)$ (see Section~\ref{sec:setting}).
\end{theorem}

We make three comments before proving the theorem.

\paragraph{On the optimality of the bound} First note that $a_d$ depends exponentially on the dimension~$d$. While removing such exponential dependence completely is challenging without stronger assumptions on $f$ (if not impossible), the constant $65$ was not optimized and could be improved. Besides, the quantity $\bigl(1-\Lip(f)/L\bigr)^d$ vanishes as $L$ approaches~$\Lip(f)$.
Importantly, the case $L=\Lip(f)$ is not really relevant in practice, because it scarcely happens that one knows exactly the best Lipschitz constant $\Lip(f)$ without knowing the function itself.
In the more realistic case when one only knows a strict upper bound $L$ on $\Lip(f)$, and under the mild assumption $\sup_{\alpha>0} c(\alpha)/c(2\alpha)<+\infty$ (which holds, e.g., if $c(\alpha)$ is polynomial in $1/\alpha$), Theorems~\ref{thm:upper_bound} and~\ref{thm:lower_bound} imply that \mfdoo{} is nearly optimal (among all certified algorithms) in terms of cost complexity, up to logarithmic and dimension-dependent multiplicative factors.

\paragraph{Earlier lower bounds} Similarly to Section~\ref{sec:mfdoo}, Theorem~\ref{thm:lower_bound} can be compared to (at least) two types of existing lower bounds. First, our lower bound generalizes that of \cite[Theorem~2]{bachoc2021instance} (where $f$ can be evaluated perfectly at the same cost as coarse evaluations) to the multi-fidelity setting, where costs play a crucial role. Note that a study of the boundary case $L=\Lip(f)$ was provided by \cite[Section~4]{bachoc2021instance}, with different phenomena appearing in dimensions $d=1$ or $d\geq 2$. Though out of the scope of this paper and with limited practical consequences, it would be interesting to investigate whether similar phenomena occur in our multi-fidelity setting. 

A second type of lower bound (of a minimax form) was proved in \cite[Theorem~1]{fiegel2020adaptive} for \emph{non-certified} algorithms, under several assumptions on the cost function (more precisely, on a so-called \emph{cost-to-bias} function) and a near-optimality dimension of $f$. Unlike the minimax approach, our lower bound is $f$-dependent. This is possible since we work with \emph{certified} algorithms, whose data-driven certificates must be robust to yet unobserved values of $f$. 

Note however that, since we require certified algorithms to output valid certificates for all $L$-Lipschitz functions $f$, Theorem~\ref{thm:lower_bound} does not imply minimax lower bounds for non-certified algorithms over smaller function classes.\footnote{This is similar in spirit to the remark on impossible adaptivity to smoothness in Section~\ref{sec:relatedworks}. A simple counter-example is given by the set $\mathcal{F}$ of all constant functions on $\mathcal{X}$, with a cost $c(\alpha)=1$ for all $\alpha$. In this case, non-certified algorithms need only $1$ evaluation of $f$ to output a maximizer in the worst case, while the lower bound of Theorem~\ref{thm:lower_bound} is of the order of $(L/\eps)^d/\ln(1/\eps)$ for small $\eps>0$ (see Appendix~\ref{sec:propertiesExamples}). Interestingly though, function classes indexed by some near-optimality dimension as in \cite{fiegel2020adaptive} may not be a good counter-example (since in that case the cost of certification $\Neps c\left(\beta\eps\right)$ can be comparable to the other terms in $S_{\beta, L}(f, \eps)$).} The two types of lower bounds can be compared on the set $\mathcal{F}_L$ of all $L$-Lipschitz functions, or on any other subset $\mathcal{F}$ if we relax the certification requirement. We briefly explain why. Let $\mathcal{F} \subset \mathcal{F}_L$, and define $\mathcal{F}$-certified algorithms similarly to Section~\ref{sec:setting}, but by only requiring certificates $\xi_t$ to be valid for all functions $f \in \mathcal{F}$ (instead of $f \in \mathcal{F}_L$). Assume also that $\inf_{\alpha >0} c(\alpha)>0$ (which is the case if, e.g., $c(\eps_0)>0$ and $c(\alpha)=c(\eps_0)$ for all $\alpha \geq \eps_0$). We claim that
\[
\sup_{f \in \mathcal{F}} \inf_A \sup_{E \in \mathcal{E}(f)} \sprec \leq \inf_{A'} \sup_{f \in \mathcal{F}} \sup_{E \in \mathcal{E}(f)} \sigma'(A',E,\eps) \;,
\]
where $A$ ranges over $\mathcal{F}$-certified algorithms, $A'$ over algorithms without certificates, and
\[
\sigma'(A',E,\eps) = \inf \biggl\{ C \in \R \sep \exists \tau \in \N^*, \sum_{t=1}^\tau c(\alpha_t(E)) \leq C \text{ and } \forall t \geq \tau, \max(f) - f(x^*_t(E)) \leq \eps \biggr\} .
\]
Note that $\sigma'(A',E,\eps)$ differs from \eqref{eq:defcostcomplexity} in that the optimization error $\sup_{t \geq \tau} f(x^\star) - f(x^*_t(E))$ replaces the certificate $\xi_{\tau}(E)$. It corresponds to the smallest total cost needed for $A'$ (when run against $E$) to output $\eps$-optimal recommendations from some time onwards.
The claimed inequality follows from two main arguments. First, $\sup_{f \in \mathcal{F}} \inf_A \phi(f,A) \leq \inf_A \sup_{f \in \mathcal{F}} \phi(f,A)$. Second, for any non-certified algorithm $A'$, we can define $\mathcal{F}$-valid certificates $\xi_t$ as follows. Let $C = \sup_{f \in \mathcal{F}} \sup_{E \in \mathcal{E}(f)} \sigma'(A',E,\eps) + \rho$ for some $\rho>0$. Then, for any round $t$, we set $\xi_t = \eps$ if $\sum_{s=1}^{t+1} c(\alpha_s(E)) > C$ (which is known at the end of round $t$), or $\xi_t = \eps_0$ otherwise. We can check that $\xi_t(E) \geq f(x^\star) - f(x^*_t(E))$ for all $t\geq 1$, $f \in \mathcal{F}$ and $E \in \mathcal{E}(f)$. Furthermore, the $\mathcal{F}$-certified algorithm $A$ obtained by endowing $A'$ with the $\xi_t$'s is such that $\sup_{f,E}\sprec \leq C$. Taking infima over $A$ and $A'$, and letting $\rho \to 0$ concludes the proof of the inequality.

\paragraph{On more collaborative environments} The lower bound of Theorem~\ref{thm:lower_bound} holds for the worst case among all environments. However, the cost complexity can be improved for some specific environments.
Indeed one could think of the following collaborative environment: when asked two times for an approximation of $f(x)$ with two accuracies $\alpha$ and $\alpha'$ at the same $x \in \mathcal{X}$, it first returns $f(x) - \alpha$ and then $f(x)+\alpha'$.
Then even with $\alpha=\alpha'=\eps_0$, the algorithm has an exact knowledge of $f(x)$ after only two queries at the same $x$.
Against such an environment, we would thus be in the same setting as in \cite{bachoc2021instance} (perfect evaluations of $f$) with only twice as many queries, and could therefore achieve a cost complexity of the order of $c(\eps_0) \cdot\bigr( \Neps + \sum_{k=1}^{m_\eps} \Ni{k}\bigl)$.
Since $c(\eps_0)$ can be much smaller than $c(\eps)$ in practice, this would greatly improve over the upper bound of Thm~\ref{thm:upper_bound}, which (by Thm~\ref{thm:lower_bound}) is nearly optimal when considering worst-case environments $E \in \mathcal{E}(f)$. In practice we could expect the environment to lie between the collaborative and worst-case extremes. The question of deriving environment-dependent lower and upper bounds is left for future work. \\

The proof of Theorem~\ref{thm:lower_bound} is inspired from that of \cite[Theorem~2]{bachoc2021instance} who addressed the case of perfect evaluations of $f$. Our generalization to the multi-fidelity setting however requires additional technicalities. Before the proof, we introduce several useful quantities and lemmas. Recall that $\mathcal{F}_L$ denotes the set of all $L$-Lipschitz functions from $\mathcal{X}$ to $\R$. We first define the quantity $\errt{A}$ for any $\tau \geq 1$, as the best certificate $\xi_\tau$ that algorithm~$A$ could output given the sequence $(x_t,\alpha_t, y_t)_{t \leq \tau}$ and given $x^*_\tau$ (note that we consider all $L$-Lipschitz functions $g:\mathcal{X} \to \R$ that are compatible with the observations $(y_t)_{t \leq \tau}$):
\[
\errt{A} = \sup\Bigl\{ \max(g) - g(x_\tau^*) \sep g\in \mathcal{F}_L \text{ and } \forall t\leq \tau, g(x_t) \in [y_t - \alpha_t, y_t+\alpha_t]\Bigr\} \;.
\]
As can be seen from the next lemma, for any certified algorithm $A$, its certificate $\xi_\tau$ at any time $\tau \geq 1$ is bounded from below by $\errt{A}$.
The proof is postponed to Appendix \ref{sec:missing_proofs}.

\begin{lemma}
  \label{lemma:certificate}
  Let $f:\mathcal{X} \to \R$ be an $L$-Lipschitz function, $E \in \mathcal{E}(f)$ be an environment and  $A$ be a certified\footnote{See the comment at the end of the statement of Theorem~\ref{thm:lower_bound}.} algorithm with certificates $\bigl(\xi_t(E)\bigr)_{t \geq 1}$ when run against $E$.\footnote{As mentioned in the introduction, $\xi_t$ is a function of all values $y_1, \ldots, y_t$ observed so far. We stress the (implicit) dependency on $E$ since it is key in the proof.}
  Then for all $\tau \in \N^*$, $\xi_\tau(E) \geq \errt{A}$.
\end{lemma}

Denoting the set of all certified algorithms by $\mathcal{A}$, we can now define\footnote{To see that the set of $C$'s is never empty, take $A=\mfdoo{}$ and apply Theorem~\ref{thm:upper_bound} and Lemma~\ref{lemma:certificate}.}
\[
\cinf = \inf\!\left\{\!C \in \R \sep \exists A \in \mathcal{A}, \forall E \in \mathcal{E}(f), \exists \tau \in \N^*, \sum_{t=1}^\tau c(\alpha_t(E)) \leq C \text{ and } \errt{A} \leq \eps \!\right\}
\]
which represents the minimum cost that certified algorithms must incur to maximize $f$ with an error certifiably below $\eps$ against any environment.
This intuitive fact is formalized in the following lemma, which is proved in Appendix~\ref{sec:missing_proofs}.

\begin{lemma}
  \label{lemma:cinf}
  Let $A$ be a certified algorithm, $f:\mathcal{X} \to \R$ be an $L$-Lipschitz function, and $\eps \in (0, \eps_0/2)$. Then, $\sup_{E \in \mathcal{E}(f)} \sprec \geq \cinf$.
\end{lemma}

Another intuitive result is that an algorithm cannot output a certificate $\xi_t \leq \eps$ unless it has already requested some value of $f$ with an evaluation accuracy $\alpha_t \leq \eps$.
This implies that the total cost needed to certify an error at level $\eps$ must be at least of $c(\eps)$.
This is stated formally below and proved in Appendix \ref{sec:missing_proofs}. Interestingly, this result would not hold if we worked with specific, possibly collaborative, environments (see a remark above).

\begin{lemma}
  \label{lemma:costeps0}
  Assume $\mathcal{X} \subset \R^d$ is compact and connected, and $c: \R_+ \to \R_+$ non-increasing.
  Then $\cinf \geq c(\eps)$ for any $L$-Lipschitz function $f:\mathcal{X} \to \R$ and any $\eps \in (0, \eps_0/2)$.
\end{lemma}

We can now prove Theorem~\ref{thm:lower_bound}.

\begin{proof}[Proof of Theorem~\ref{thm:lower_bound}]
  We assume without loss of generality that $\Lip(f) < L$ and set $\Omega_f = \bigl(\frac{1-\Lip(f)/L}{65}\bigr)^d$.
  We want to show that   \[ \sup_{E \in \mathcal{E}(f)} \sprec \geq \frac{\Omega_f}{1+m_\eps} S_{16, L}(f, \eps)\;. \]

  Since $\sup_{E \in \mathcal{E}(f)} \sprec \geq \cinf$ (by Lemma \ref{lemma:cinf}), it is sufficient to show that $\cinf \geq \frac{\Omega_f}{1+m_\eps}S_{16, L}(f, \eps)$. We set $K = \frac{16L}{L - \Lip(f)}$ and note that $\Omega_f \leq \frac{1}{(1+4K)^d}$. We can distinguish between two cases:\\[0.2cm]
  \textit{First case:} Assume first that $\frac{S_{16, L}(f, \eps)}{1+m_\eps} \leq (1+4K)^dc(\eps)$.
  Then
  \[
  \frac{\Omega_f}{1+m_\eps}S_{16, L}(f, \eps) \leq c(\eps) \leq \cinf \;,
  \]
  where the last inequality follows from Lemma \ref{lemma:costeps0}.
  In this case, the theorem is proved.\\[0.2cm]
  \textit{Second case:} We now assume that $\frac{S_{16, L}(f, \eps)}{1+m_\eps} > (1+4K)^d c(\eps)$.
  The idea is to upper bound the average of the $(1+m_\eps)$ terms that define $S_{16, L}(f, \eps)$ by the largest one.

  Let $\teps$ be the scale with maximum contribution in \eqref{eq:Sbeta} with $\beta=16$, that is:
  \[ \teps = \left\{
      \begin{matrix}
        \eps & \text{ if } \Neps c(16\eps) \geq \max_{1 \leq k \leq m_\eps} \Ni{k} c(16\eps_k)\\
        \eps_{k^*-1} & \text{ otherwise, where } k^* \in \argmax_{ 1 \leq k \leq m_\eps} \Ni{k} c(16\eps_k)
      \end{matrix}
    \right.
  \]
  Since $\Neps \leq \mathcal{N}\left(\mathcal{X}_{\eps}, \frac{\eps}{2L}\right)$ and $\Ni{k} \leq \mathcal{N}\left(\mathcal{X}_{\eps_{k-1}}, \frac{\eps_{k-1}}{2L}\right)$ for all $1 \leq k \leq m_\eps$, and since $c$ is non-increasing, we then have $S_{16, L}(f, \eps) \leq (1+m_\eps) \mathcal{N}\left(\mathcal{X}_{\teps}, \frac{\teps}{2L}\right) c(8\teps)$.\\
  Then, using Lemma \ref{lemma:packing_number}, the previous result, and the assumption of the second case, we have: %
  \begin{align}
    \nonumber \mathcal{N}\left(\mathcal{X}_{\teps}, \frac{K\teps}{L}\right)c(8\teps)
    \nonumber &\geq \left(\frac{1}{1+4K}\right)^d \mathcal{N}\left(\mathcal{X}_{\teps}, \frac{\teps}{2L}\right)c(8\teps) \\
    \label{eq:S8} &\geq \left(\frac{1}{1+4K}\right)^d \frac{S_{16, L}(f, \eps)}{(1+m_\eps)} > c(\eps)\;.
  \end{align}

  To prove our result, we assume for a moment that $\cinf < \frac{\Omega_f}{1+m_\eps}S_{16, L}(f, \eps)$ and will show that it raises a contradiction.
  Combining with \eqref{eq:S8} and $\Omega_f \leq \frac{1}{(1+4K)^d}$, this indeed yields
  \[ \cinf < \Omega_f (1+4K)^d \mathcal{N}\left(\mathcal{X}_{\teps}, \frac{K\teps}{L}\right)c(8\teps) \leq \mathcal{N}\left(\mathcal{X}_{\teps}, \frac{K\teps}{L}\right)c(8\teps)\;.\]
  Then, by definition of $\cinf$, there exists an algorithm $A \in \mathcal{A}$ such that for all environments $E\in \mathcal{E}(f)$, there exists $\tau \in \N^*$ such that
  \begin{equation}
    \label{eq:cinflower}
    \sum_{t=1}^\tau c(\alpha_t(E)) < \mathcal{N}\left(\mathcal{X}_{\teps}, \frac{K\teps}{L}\right)c(8\teps) \text{ and }\errt{A} \leq \eps\;.
  \end{equation}

  We now consider the ``noiseless'' environment $E=(E_t)_{t\geq 1} \in \mathcal{E}(f)$ defined by  $E_t(x, \alpha) = f(x)$ for all $t\geq 1$, $x\in \mathcal{X}$, and $\alpha > 0$. 
  Let $\tau \in \N^*$ be such that \eqref{eq:cinflower} holds.
  Let $M =  \mathcal{N}\left(\mathcal{X}_{\teps}, \frac{K\teps}{L}\right)$, and let $\{\tilde{x}_1, \ldots, \tilde{x}_M\}$ be a $(K\teps /L)$-packing of $\mathcal{X}_{\teps}$.
  Note that the closed balls $B(\tilde{x}_m, \frac{K\teps}{2L})$ with centers $\tilde{x}_1, \ldots, \tilde{x}_M$ and radius $K\teps / 2L$ are pairwise disjoint.
  Note also that $M \geq 2$ from $\mathcal{N}\left(\mathcal{X}_{\teps}, \frac{K\teps}{L}\right) > \frac{c(\eps)}{c(8\teps)} \geq 1$ by \eqref{eq:S8}, $8\teps \geq \eps$, and $c$ being non-increasing ($c(8\teps)>0$ by \eqref{eq:S8}).

  For any $1\leq m \leq M$, let $c_m$ be the maximum cost spent at any round on the $m$-th ball, that is $c_m = \max \{ c(\alpha_t) \sep t \in \mathcal{T}_m\}$ if the set $\mathcal{T}_m := \{  t = 1,\ldots,\tau \sep \norm{x_t(E) - \tilde{x}_m} \leq K\teps /2L\}$ is non-empty, and $c_m = 0$ otherwise.
  We know from \eqref{eq:cinflower} that the total cost up to round $\tau$ is smaller than $M c(8\teps)$.
  By the pigeonhole principle, there is at least one $m \leq M$ for which $c_m < c(8\teps)$.
  Assume without loss of generality that this is true for $m=1$.
  Then for any $t \in \mathcal{T}_1$ (if such $t$ exists), the cost $c\left(\alpha_t(E)\right)$ is smaller than $c(8\teps)$.
  Therefore, either $\alpha_t(E) > 8 \teps$ whenever the ball $B\left(\tilde{x}_1, \frac{K\teps}{2L}\right)$ is visited (since $c$ is non-increasing) or this ball is never visited.

  We just showed that, on the ball $B(\tilde{x}_1, \frac{K\teps}{2L})$, algorithm $A$ never queried $f$ with an evaluation accuracy $\alpha_t \leq 8 \teps$. Next we show the following consequence: that the inequality $\errt{A} \leq \eps$ in \eqref{eq:cinflower} cannot be true, by exhibiting an $L$-Lipschitz function $g\in \mathcal{F}_L$ compatible with the observations $y_t = E_t\bigl(x_t(E), \alpha_t(E)\bigr) = f(x_t(E))$ and such that $\max(g) - g(x_\tau^*(E)) > \eps$. This will raise a contradiction in \eqref{eq:cinflower} and conclude the proof. To that end, we consider the two functions $g = f \pm h_{\teps}$, with $h_{\teps}:\mathcal{X} \to \R$ defined by
  \vspace{-0.2cm}
  \begin{equation}
  \label{eq:geps}
  h_{\teps}(x) = \max\left\{8 \teps - 16 \frac{L}{K}\norm{x - \tilde{x}_1}, 0\right\}\;.
  \end{equation}
  First note that both $f - h_{\teps}$ and $f+h_{\teps}$ are $L$-Lipschitz, since $h_{\teps}$ is $(L-\Lip(f))$-Lipschitz (by $\frac{16L}{K} = L-\Lip(f)$).
  Moreover, since $h_{\teps}$ is supported on $\mathcal{X} \cap B(\tilde{x}_1, K\teps/2L)$ and $\norm{h_{\teps}}_\infty \leq 8\teps \leq \alpha_t(E)$ for all $t \in \mathcal{T}_1$, the two functions $f-h_{\teps}$ and $f + h_{\teps}$ belong by construction to the set
    \begin{equation*}
  \mathcal{G} := \Bigl\{g\in \mathcal{F}_L \sep \forall t=1,\ldots,\tau, \; g(x_t(E)) \in \bigl[f(x_t(E)) - \alpha_t(E), f(x_t(E)) + \alpha_t(E)\bigr] \Bigr\}\;.
  \end{equation*}
  We now show that $\max(g) - g(x_\tau^*(E)) > \eps$ for $g = f - h_{\teps}$ or $g = f + h_{\teps}$, by distinguishing two subcases.
  If $x_\tau^*(E) \in B(\tilde{x}_1, K\teps / 4L)$, we perturb $f$ ``downwards'' around $\tilde{x}_1$ and consider $g = f - h_{\teps}$. In this case, since $h_{\teps}(x^*_\tau(E)) \geq 4\teps$ and $h_{\teps}(\tilde{x}_2) =  0$, we have $\max(g) - g\bigl(x^*_\tau(E)\bigr) \geq f(\tilde{x}_2) - h_{\teps}(\tilde{x}_2) - (f(x^*_\tau(E)) - h_{\teps}(x^*_\tau(E))) \geq -\teps + 4\teps = 3\teps$.\footnote{We used the fact that $f(\tilde{x}_m) - f(x^*_\tau(E)) \geq f(\tilde{x}_m) - \max(f) \geq -\teps$ for all $1 \leq m \leq M$ (since $\tilde{x}_m \in \mathcal{X}_{\teps}$).} In the other case, if $x_\tau^*(E) \notin B(\tilde{x}_1, K\teps/4L)$, we consider $g = f + h_{\teps}$: since $h_{\teps}(x^*_\tau(E)) \leq 4\teps$ and $h_{\teps}(\tilde{x}_1) = 8\teps$, we have
  $\max(g) - g\bigl(x^*_\tau(E)\bigr) \geq f(\tilde{x}_1) + h_{\teps}(\tilde{x}_1) - (f(x^*_\tau(E))+h_{\teps}(x^*_\tau(E))) \geq - \teps + 8\teps - 4\teps = 3\teps$.
  
  \noindent
  In both subcases above, we proved $\max(g) - g\bigl(x^*_\tau(E)\bigr) \geq 3\teps > \eps$ for some $g \in \{f - h_{\teps}, f + h_{\teps}\} \subset \mathcal{G}$, which entails $\errt{A} >\eps$ (by definition of $\textrm{err}_\tau$). This raises a contradiction in  \eqref{eq:cinflower}, so that we must have $\cinf \geq \frac{\Omega_f}{1+m_\eps} S_{16, L}(f, \eps)$. This concludes the proof.
\end{proof}

\section{Special Case: Noisy Evaluations of \texorpdfstring{$f$}{f} (a.k.a. Stochastic Setting)}
\label{sec:stochastic}

Previously all the environments $E$ that we considered were deterministic.
We now assume that the algorithm receives noisy (stochastic and unbiased) evaluations of $f$, but that for all $t \in \N^*$ it can observe several independent noisy evaluations of $f(x_t)$ and decide the number $m_t$ of them.

More formally, we consider the following variant of the online protocol described in Section~\ref{sec:setting}. Let $(\zeta_{t,u})_{t,u \in \N^*}$ be a sequence of independent $v$-subGaussian random variables.\footnote{A real-valued random variable $X$ is $v$-subGaussian if $\E[\exp(\lambda X)] \leq \exp(\lambda^2 v/2)$ for all $\lambda \in \R$. In particular, $\E[X]=0$ and $\textrm{Var}[X] \leq v$. Two examples are the Gaussian distribution $\mathcal{N}(0,v)$ and the uniform distribution $\textrm{Unif}\bigl([-\sqrt{v},\sqrt{v}]\bigr)$.} 
The $\zeta_{t,u}$'s are unknown, but the constant $v > 0$ is assumed to be known to the learner.
At each round $t \in \N^*$, the algorithm $A$ chooses a query point $x_t \in \mathcal{X}$, as well as a number $m_t \geq 1$ of noisy evaluations (instead of $\alpha_t$). The algorithm incurs a cost equal to $m_t$. In return, the environment outputs a mini-batch $(y_{t,1}, \ldots, y_{t,m_t})$ with $m_t$ components (instead of a single inaccurate evaluation $y_t$), where $y_{t,u} = f(x_t) + \zeta_{t,u}$ for any $1 \leq u \leq m_t$. Then, just as before, $A$ outputs a recommendation $x_t^* \in \mathcal{X}$ for the maximum of $f$, together with a (tentative) error certificate $\xi_t \geq 0$. In this setting, the goal is (with high probability) to maximize $f$ with an error certifiably below $\eps$, while minimizing the total number of evaluations of $f$. This problem corresponds to $\eps$-best arm identification in Lipschitz bandits with a continuous set $\cX$ of arms (see details in Section~\ref{sec:relatedworks}).

\paragraph{Algorithm} We reduce this problem to the deterministic setting of Section~\ref{sec:mfdoo}. We consider \mfsoo{} (Certified Multi-Fidelity Stochastic Optimistic Optimization), which is a mini-batch version of \mfdoo{} and whose pseudo-code is given in Algorithm~\ref{algo:stoch-doo} below. In the sequel, we use the same identification between nodes and rounds as before (see Footnote~\ref{ft:injection}).

The intuition behind \mfsoo{} is the following: for a mini-batch of size $m_t$, the average $y_t \defeq \frac{1}{m_t}\sum_{u=1}^{m_t} y_{t,u}$ is an unbiased estimate of $f(x_t)$ with a variance bounded by $v/m_t$.
Therefore (see later for more details), with high probability, the absolute difference $|y_t-f(x_t)|$ is at most roughly of the order of the standard deviation $\sqrt{v/m_t}$.
To be in a special case of Section~\ref{sec:mfdoo}, Algorithm~\ref{algo:stoch-doo} makes $\sqrt{v/m_t}$ comparable to the evaluation accuracy $\alpha_t$ that \mfdoo{} would request, by choosing $m_t \approx v/\alpha_t^2$. This will allow us to apply Theorem~\ref{thm:upper_bound} with a cost $c(\alpha) \approx v/\alpha^2$.

To make the above intuition more rigorous (multiple high probability bounds will be used simultaneously), we use a careful weighted union bound on the nodes of the hierarchical partitioning tree. For some desired risk level $\gamma \in (0,1)$, writing $(h_t,i_t)$ for the node evaluated at time $t \geq 1$ (Line~\ref{eq:def_ht_it} if $t \geq 2$), we take
\begin{equation}
\label{eq:mt}
m_t = \left\lceil \frac{2v}{\alpha_t^2} \ln\left(\frac{2}{\gamma_{h_t, i_t}}\right)\right\rceil \;, \; \textrm{with} \; \gamma_{h, i} = \frac{\gamma}{(h+1)(h+2)K^h} \textrm{ for $h \in \mathbb{N}$ and $i \in \{0,\ldots,K^h-1\}$}.
\end{equation}
Note that the weights sum up to $\sum_{h=0}^{+\infty} \sum_{i=0}^{K^h-1} \gamma_{h,i} = \gamma$.

\begin{algorithm}
  \caption{\mfsoo{} (Certified Multi-Fidelity Stochastic Optimistic Optimization)}
  \label{algo:stoch-doo}
  \hspace*{\algorithmicindent} \textbf{Inputs}: $\mathcal{X}$, $K$, $(X_{h,i})_{h \in \N, i \in \{0,\ldots, K^h-1\}}$, $(x_{h,i})_{h \in \N, i \in \{1,\ldots, K^h-1\}}$, $\delta$, $R$, $L$, $v$, and $\gamma$\\
  \hspace*{\algorithmicindent} \textbf{Initialization} Let $t\leftarrow 1$, $(h_1,i_1) \leftarrow (0,0)$, and $\mathcal{L}_1 \leftarrow \{(0,0)\}$
  \begin{algorithmic}[1]
    \STATE Pick the first query point $x_1 \leftarrow x_{0,0}$, accuracy $\alpha_1 \leftarrow LR$, prior value $\gamma_{0,0} \leftarrow \gamma/2$, and evaluation number $m_1 \leftarrow \left\lceil \frac{2v}{(LR)^2} \ln\left(\frac{4}{\gamma}\right)\right\rceil$
    \STATE Observe the noisy evaluations  $(y_{1,u})_{1 \leq u \leq m_1} = (f(x_1) + \zeta_{1,u})_{1 \leq u \leq m_1}$
    \STATE Compute $y_1 = \frac{1}{m_1} \sum_{u=1}^{m_1} y_{1,u}$
    \STATE Output the recommendation $x_1^* \leftarrow x_1$ and certificate $\xi_1 \leftarrow LR$
    \STATE Pick the first node $(h^*, i^*) \leftarrow (0, 0)$
    \FOR{$iteration=1, 2, \ldots$}{
      \FORALL{child $(h^*+1, j)$ of $(h^*,i^*)$}{
          \IF{$\mathcal{X}_{h^*+1,j} \cap \mathcal{X} \neq \varnothing$}
          \STATE Let $t \leftarrow t+1$, $(h_t, i_t) \leftarrow (h^*+1, j)$ and $\mathcal{L}_t \leftarrow \mathcal{L}_{t-1} \cup \{(h_t,i_t)\}$ \label{eq:def_ht_it}
          \STATE Pick the next query point $x_t \leftarrow x_{h_t,i_t}$ and accuracy $\alpha_t \leftarrow LR\delta^{h_t}$
          \STATE Compute the prior value $\gamma_{h_t, i_t} = \gamma/((h_t+1)(h_t+2)K^{h_t})$
          \STATE Pick the number of evaluations $m_t \leftarrow \left\lceil \frac{2v}{\alpha_t^2} \ln\left(\frac{2}{\gamma_{h_t, i_t}}\right)\right\rceil$  \label{eq:choice_mt}
          \STATE Observe the noisy evaluations  $(y_{t,u})_{1 \leq u \leq m_t} = (f(x_t) + \zeta_{t,u})_{1 \leq u \leq m_t}$
          \STATE Compute $y_t = \frac{1}{m_t} \sum_{u=1}^{m_t} y_{t,u}$
          \STATE Output the recommendation $x_t^* = x_{\tilde{t}},$ with $\tilde{t} \in \argmax_{1 \leq s \leq t} \{y_s - \alpha_s\}$
          \STATE Output the certificate $\xi_t = y_{h^*, i^*} + LR\delta^{h^*} + \alpha_{h^*, i^*} - (y_{\tilde{t}} -\alpha_{\tilde{t}})$ \label{eq:stoch-certificate}
          \ENDIF
      }
      \ENDFOR
      \STATE Remove $(h^*, i^*)$ from $\mathcal{L}_t$
      \STATE Let $(h^*, i^*) \in \argmax_{(h, i) \in \mathcal{L}_t} \{y_{h,i} + LR\delta^h+\alpha_{h,i}\}$ \label{eq:stoch-argmax}
      \STATE Update the last certificate $\xi_t = y_{h^*, i^*} + LR\delta^{h^*} + \alpha_{h^*, i^*} - (y_{\tilde{t}} -\alpha_{\tilde{t}})$
    }
    \ENDFOR
  \end{algorithmic}
\end{algorithm}

\paragraph{Sample complexity} 
Define the stopping time
\[
\tau(f,\eps) \defeq \inf \{ t \geq 1 \sep \xi_t \leq \eps \}\;.
\]
We now bound the total number $\sum_{t=1}^{\tau(f,\eps)} m_t$ of evaluations of $f$ that \mfsoo{} requests before certifying an $\eps$-maximizer of $f$. The next high-probability bound (Proposition~\ref{prop:stoch_upper_bound}) is in terms of the quantity $S_{\beta, L}(f, \eps)$ defined in \eqref{eq:Sbeta} with the cost function $c = c_\gamma$ given by
\begin{equation}
  \label{eq:cgamma}
  c_\gamma(\alpha) = \left\lceil \frac{2v}{\alpha^2} \ln\left(\frac{2(h(\alpha)+1) (h(\alpha)+2) K^{h(\alpha)}}{\gamma}\right)\right\rceil, \quad \text{where} \quad h(\alpha) = \frac{\ln(LR/\alpha)}{\ln(1/\delta)}\;.
\end{equation}

\noindent
We recall that $\eps_0 = L \, \diam(\mathcal{X})$, $m_\eps = \left\lceil\log_2(\eps_0/\eps) \right\rceil$, $\eps_{m_\eps} = \eps$ and $\eps_k = \eps_0 2^{-k}$ for $1 \leq k \leq m_\eps-1$.

\begin{proposition}
  \label{prop:stoch_upper_bound}
  Suppose that $\mathcal{X} \subset \R^d$ is compact, that Assumptions~\ref{assum:diameter} and~\ref{assum:nu} hold, and denote by $a>0$ the same constant as in Theorem~\ref{thm:upper_bound}. Then, in the stochastic setting described above, for any known constants $L, v>0$ and $\gamma \in (0, 1)$, for any $L$-Lipschitz function $f: \mathcal{X} \to \R$ with maximizer $x^\star \in \mathcal{X}$, and any $\eps \in (0,\eps_0)$, \mfsoo{} (Algorithm \ref{algo:stoch-doo}) satisfies
  \[
  \mathbb{P}\left(
      \Bigl[\forall t \geq 1, f(x^\star) - f(x^*_t) \leq \xi_t\Bigr]
      \text{ and }
      \sum_{t=1}^{\tau(f, \eps)} m_t \leq a S_{\frac{\delta}{3}, L}(f, \eps) + \left\lceil \frac{2v}{(LR)^2} \ln\left(\frac{4}{\gamma}\right)\right\rceil\right) \geq 1 - \gamma \;,
  \]
  where the probability is taken over the noise sequence $(\zeta_{t,u})_{t,u \in \N^*}$, and where $S_{\delta/3, L}(f, \eps)$ is the quantity defined in \eqref{eq:Sbeta} with $\beta=\delta/3$ and $c = c_\gamma$ (see \eqref{eq:cgamma} above).
\end{proposition}

The proof appears below; we start with two comments. First, Proposition~\ref{prop:stoch_upper_bound} implies that, with high probability, the $\xi_t$'s are valid certificates and the total number of evaluations of $f$ that \mfsoo{} requests before certifying an $\eps$-maximizer of $f$
is bounded roughly by (combining \eqref{eq:Sbeta} with \eqref{eq:cgamma}, and omitting log factors and some dimension-dependent constants)
\begin{align*}
\frac{v \, \Neps}{\eps^2}  + \sum_{k=1}^{m_\eps} \frac{v \, \Ni{k}}{\eps_k^2} \; \approx \; v \, L^d \int_{\mathcal{X}} \frac{d x}{(f(x^\star)-f(x)+\eps)^{d+2}} \;,
\end{align*}
where the sum-integral approximation (which omits multiplicative constants) holds under the following mild geometric condition on $\mathcal{X}$: there exist constants $r_0>\diam(\cX)/2$, $\rho\in(0,1]$ such that for all $x\in\cX$ and $r\in(0,r_0]$, $\vol(B(x,r) \cap \cX) \geq \rho \, \vol(B(x,r))$. This condition roughly states that $\mathcal{X}$ has a non-negligible volume locally everywhere (e.g., we can take $r_0=1$ and $\rho=2^{-d}$ if $\cX=[0,1]^d$ and $\norm{\cdot}$ is the sup norm). The proof of this sum-integral approximation follows essentially from \cite[Theorem~1]{bachoc2021instance}, with a direct extension to non-increasing costs.

Therefore, a consequence of Proposition~\ref{prop:stoch_upper_bound} is that, under a mild condition on $\mathcal{X}$, the cost complexity in the stochastic setting is roughly proportional to $\int_{\mathcal{X}} d x / (f(x^\star)-f(x)+\eps)^{d+2}$, as conjectured by \cite{bachoc2021instance}.

Second, to the best of our knowledge, Proposition~\ref{prop:stoch_upper_bound} provides the first sample complexity bound for an $(\eps,\gamma)$-PAC algorithm in continuum-armed Lipschitz bandits. The case in which $\eps=0$ and $\cX$ is finite was addressed by \cite[Appendix~F]{wang21-fastPureExplorationFranckWolfe}. Note however that we bound a $(1-\gamma)$-quantile of the total number of evaluations of $f$, instead of its expectation (a more classical quantity in the best arm identification literature). This alternative result, together with the question of deriving an instance-dependent lower bound in the stochastic setting (multiplicative constants will most likely differ from the upper bound), are left for future work.\\

To prove Proposition~\ref{prop:stoch_upper_bound}, we use the following classical lemma, which helps reduce the stochastic setting with mini-batches to the deterministic setting with inaccurate evaluations. (We will later combine this lemma with $y_t - f(x_t) = \frac{1}{m_t} \sum_{u=1}^{m_t} (y_{t,u} - f(x_t)) = \frac{1}{m_t} \sum_{u=1}^{m_t} \zeta_{t,u}$.)

\begin{lemma}
  \label{lemma:high-proba}
  Let $(\zeta_{t,u})_{t,u \in \N^*}$ be a sequence of independent $v$-subGaussian random variables for some $v > 0$. Let $\gamma \in (0, 1)$, and let $(\alpha_t)_{t \geq 1}$ and $(h_t,i_t)_{t \geq 1}$ be two predictable\footnote{That is, we assume $\alpha_t$, $h_t$ and $i_t$ to be measurable w.r.t. the subsequence $(\zeta_{s,u})_{1 \leq s \leq t-1, u \in \N^*}$.} sequences such that, almost surely, $\alpha_t > 0$, $h_t \in \N$ and $i_t \in \{0,\ldots,K^{h_t}-1\}$ for all $t \geq 1$, and $t \geq 1 \mapsto (h_t,i_t)$ being injective. Then, for $m_t$ and $\gamma_{h,i}$ defined as in \eqref{eq:mt}, we have
  \[
  \mathbb{P}\left(\forall t \in \N^*, \abs{\frac{1}{m_t} \sum_{u=1}^{m_t} \zeta_{t,u}} < \alpha_t\right) \geq 1- \gamma \;.
  \]
\end{lemma}

\begin{proof}
  Let $\mathcal{F}_t$ denote the $\sigma$-field generated by the random variables $\zeta_{s,u}$, $s \in \{1,\ldots,t\}$, $u \in \N^*$. (By convention, $\mathcal{F}_0$ is the trivial $\sigma$-field.) Let $t \geq 1$. Since $\alpha_t$ and $m_t$ are $\mathcal{F}_{t-1}$ measurable, and the $\zeta_{t,u}$, $u \in \N^*$, are independent and $v$-subGaussian conditionally on~$\mathcal{F}_{t-1}$,
  \begin{align*}
    \mathbb{P}\left(\abs{\frac{1}{m_t} \sum_{u=1}^{m_t} \zeta_{t,u}} \geq \alpha_t\right)
    &= \E\left[\mathbb{P}\left(\abs{\frac{1}{m_t} \sum_{u=1}^{m_t} \zeta_{t,u}} \geq \alpha_t \Big\rvert \mathcal{F}_{t-1} \right)\right] \leq \E\left[2 e^{-\frac{m_t\alpha_t^2}{2v}}\right] \leq \E\left[\gamma_{h_t, i_t}\right] \;.
  \end{align*}
  By a union bound, this yields
  \begin{align}
    \mathbb{P}\left(\exists t \in \N^*, \abs{\frac{1}{m_t} \sum_{u=1}^{m_t} \zeta_{t,u}} \geq \alpha_t\right)
    \nonumber & \leq \sum_{t=1}^{+\infty}   \mathbb{P}\left(\abs{\frac{1}{m_t} \sum_{u=1}^{m_t} \zeta_{t,u}} \geq \alpha_t\right) \leq \E\left[\sum_{t=1}^{+\infty} \gamma_{h_t, i_t}\right] \leq \gamma \;, \label{eq:high_proba_mt}
  \end{align}
  where the last inequality follows by injectivity of $t \geq 1 \mapsto (h_t,i_t)$ and  $\sum_{h=0}^{+\infty} \sum_{i=0}^{K^h-1} \gamma_{h,i} = \gamma$. Taking the complementary event concludes the proof.
\end{proof}

\begin{proof}[Proof of Proposition \ref{prop:stoch_upper_bound}]
  We now explain how to treat the problem as a special case of Section~\ref{sec:mfdoo} (deterministic yet inaccurate evaluations), with the cost function $c_\gamma$ defined in \eqref{eq:cgamma} and a well-chosen random environment.
  
  First note that the assumptions of Lemma \ref{lemma:high-proba} are met, so that with high probability the $y_t$'s are $\alpha_t$-close to $f(x_t)$ simultaneously for all $t \geq 1$. More formally, they can be seen as generated by a random environment $(E_t^\omega)_{t \geq 1}$ defined as follows.
  Let $(\Omega, \mathcal{F}, \mathbb{P})$ be the probability space on which the random variables $(\zeta_{t,u})_{t,u}$ are defined.
  For any fixed element $\omega \in \Omega$ and any $t \geq 1$, we define the function $E_t^\omega:\mathcal{X} \times \R^*_+ \to \R$ by
  \[
  E_t^\omega(x, \alpha) = \left\{ \begin{matrix} f(x) + \frac{1}{m_t(\omega)} \sum_{u=1}^{m_t(\omega)}\zeta_{t,u}(\omega) &\text{ if } \abs{\frac{1}{m_t(\omega)} \sum_{u=1}^{m_t(\omega)}\zeta_{t,u}(\omega)} \leq \alpha \\ f(x) &\textit{ otherwise.}\end{matrix}\right.
  \]
  Note that the environment $(E_t^\omega)_{t \geq 1}$ lies in $\mathcal{E}(f)$ for each $\omega \in \Omega$. 
  Now, denoting by $\tilde{\Omega}=\left\{\omega \in \Omega \sep \forall t \in \N^*, \abs{\frac{1}{m_t(\omega)} \sum_{u=1}^{m_t(\omega)} \zeta_{t,u}(\omega)} < \alpha_t\right\}$ the event considered in Lemma \ref{lemma:high-proba}, we can see that, for any $\omega \in \tilde{\Omega}$,\footnote{For the sake of readability, we drop some dependencies in $\omega$.} the value $E_t^\omega(x_t, \alpha_t) = f(x_t) + \frac{1}{m_t} \sum_{u=1}^{m_t}\zeta_{t,u} = \frac{1}{m_t} \sum_{u=1}^{m_t} \bigl(f(x_t)+\zeta_{t,u}\bigr)$ coincides with $y_t = \frac{1}{m_t} \sum_{u=1}^{m_t} y_{t,u}$ for all $t \geq 1$.
  We can thus apply Lemma \ref{lemma:certification}: for any $\omega \in \tilde{\Omega}$, we have $f(x^\star) - f(x_t^*) \leq \xi_t$ for all $t\geq 1$ ($\xi_t$ is a valid certificate).
  
  Furthermore, a call to $E_t^\omega(x_t, \alpha_t)$ requires $m_t = \left\lceil \frac{2v}{\alpha_t^2} \ln\left(\frac{2}{\gamma_{h_t, i_t}}\right)\right\rceil = c_\gamma(\alpha_t)$ noisy evaluations of $f$, with $c_\gamma$ defined in \eqref{eq:cgamma}. Putting everything together, for any $\omega \in \tilde{\Omega}$, the behavior of \mfsoo{} coincides with the behavior of \mfdoo{} against the environment $E_\omega \in \mathcal{E}(f)$. In particular the total cost $\sum_{t=1}^{\tau(f, \eps)} m_t = \sum_{t=1}^{\tau(f, \eps)} c_\gamma(\alpha_t)$ coincides with the cost complexity of \mfdoo{} against $E_\omega$, with the non-increasing cost function $c_\gamma$. We can thus use Theorem~\ref{thm:upper_bound}: for any $\omega \in \tilde{\Omega}$,
  \[
  \sum_{t=1}^{\tau(f, \eps)} m_t = \sprecAE{\mfdoo{}}{E^\omega} \leq a S_{\frac{\delta}{3}, L}(f, \eps) + c_\gamma(LR) \;.
  \]
  Recalling that $\mathbb{P}(\tilde{\Omega}) \geq 1- \gamma$ concludes the proof.
\end{proof}

\section*{Acknowledgements}
The authors would like to thank Fran\c{c}ois Bachoc for insightful feedback, as well as two anonymous reviewers who suggested interesting discussions within the paper. This work has benefited from the AI Interdisciplinary Institute ANITI, which is funded by the French ``Investing for the Future--PIA3'' program under the Grant agreement ANR-19-P3IA-0004. The authors gratefully acknowledge the support of the DEEL project.\footnote{\url{https://www.deel.ai/}}

\appendix

\section{Some properties of $S_{\beta, L}(f, \eps)$, and two simple examples}
\label{sec:propertiesExamples}

Our upper and lower bounds all involve the quantity
\[
  S_{\beta, L}(f, \eps) = \Neps c\left(\beta\eps\right) + \sum_{k=1}^{m_\eps} \Ni{k}c\left(\beta\eps_k\right) \;.
\]

Let us make a few comments to help interpret these bounds. We start by explaining how $S_{\beta, L}(f,\eps)$ depends on the scaling factor $\beta$, the Lipschitz bound $L$, and the ambient norm $\Vert \cdot \Vert$.\\[-0.2cm]
\begin{itemize}[leftmargin=0.5cm]
    \item \textbf{Scaling factor $\beta$}: note that $\beta$ only appears within the costs $c(\beta\eps)$ and $c(\beta\eps_k)$. If the non-increasing cost function $c$ satisfies the mild assumption $\sup_{\alpha>0} c(\alpha)/c(2\alpha)<+\infty$ (which holds, e.g., if $c(\alpha)$ is polynomial in $1/\alpha$), then multiplying $\beta$ by a constant factor can only change $S_{\beta, L}(f, \eps)$ by at most a multiplicative constant. 
    \item \textbf{Lipschitz bound $L$}: the bound in Theorem \ref{thm:upper_bound} holds as long as $L \geq \Lip(f)$, where $\Lip(f)$ denotes the best Lipschitz constant of $f$. Note that $S_{\beta, L}(f,\eps)$ depends on $L$ in two different ways: (i) within the packing numbers and (ii) in the definitions of $\eps_0 := L \cdot \diam(\mathcal{X})$ and thus $m_\eps := \left\lceil\log_2(\eps_0/\eps) \right\rceil$ and $\eps_k := \eps_0 2^{-k}$. The effect (ii) is mostly negligible, since for a fixed $f$, replacing $L$ with $L' = 2L$ only creates a new layer $\cX_{(\eps_0,2\eps_0]}$, which is empty. As for the effect (i), by Lemma \ref{lemma:packing_number} in the appendix, $S_{\beta, L}(f,\eps)$ cannot deteriorate by more than a factor of the order of $(L'/L)^d$ if $L$ is replaced with $L'>L$.
    \item \textbf{Ambient norm $\Vert \cdot \Vert$}: the effect is very similar to that of $L$, since the norm $\Vert \cdot \Vert$ appears both in the definition of the packing numbers and in the value of $L$ (and thus $\eps_0$, $\eps_k$ and $m_\eps$).\footnote{Note that two effects can cancel out: for instance, rescaling the norm $\Vert \cdot \Vert$ by a factor of $2^k$ (and $L$ by a factor of $2^{-k}$) leaves $S_{\beta, L}(f,\eps)$ unchanged.} By the equivalence of norms in $\R^d$ and by Lemma~\ref{lemma:packing_number} in the appendix, the same conclusions apply as in the previous item.
\end{itemize}

\ \\
We now informally compute $S_{\beta, L}(f, \eps)$ on two simple (extreme) examples. \\

\noindent
\textbf{Example 1.} Consider a constant function $f_1$ on any compact set $\cX \subset \R^d$ with nonempty interior. Let $L>0$ be fixed. Then, $S_{\beta, L}(f_1, \eps) = \cN(\cX, \eps/L)c(\beta \eps)$ which is of the order of $\left(\frac{L}{\eps}\right)^d c(\beta \eps)$ for small $\eps>0$. \\

\noindent
\textbf{Example 2.} Let $\cX = B_{\norm{\cdot}_p}(0, 1)$ be the unit closed $\ell_p$-ball in $\R^d$, with $1 \leq p \leq +\infty$. Consider the function $f_2$ defined for all $x \in \cX$ by $f_2(x)=1 - |x|$, where $|\cdot|$ is any norm in $\R^d$. Note that $f_2$ is $L$-Lipschitz w.r.t. $\norm{\cdot}_p$ for $L>0$ large enough. Next we informally write $\approx$ to mean that both sides are equal up to constants that only depend on $d$, $p$, $|\cdot|$, $L$, and $c(\cdot)$. First note that $\cX_\eps = B_{|\cdot|}(0, \eps)\cap \cX$, so that $\cN(\cX_\eps, \eps/L) \approx 1$ for small $\eps>0$ (by \cite[Lemma~5.7]{Wainwright19-HighDimensionalStatistics} and the equivalence of norms in $\R^d$; see also Lemma~\ref{lemma:packing_number} in the appendix). Similarly, $\Ni{k} \approx 1$ as soon as the set $\cX_{(\eps_k,\eps_{k-1}]} = \{x \in \cX: \eps_k < |x| \leq \eps_{k-1}\}$ is nonempty, which is the case for at least $k=1,\ldots,\lceil m_\eps/2\rceil$ if $\eps$ is small enough (by equivalence of norms). Overall,  using the fact that $c$ is non-increasing, we get $c(\beta \eps) \lesssim S_{\beta, L}(f_2, \eps) \lesssim \ln(1/\eps) c(\beta \eps)$ for small $\eps$. \\

\noindent
\textbf{Remark.} In practice we could face a function $f$ that lies in between the two extremes $f_1$ and $f_2$, such as $f(x)=1 - |x|^\nu$ with $\nu > 1$ ($f_1$ corresponds locally to $\nu = +\infty$). In this case, and if the origin lies in the interior of $\cX$, similar computations yield $S_{\beta, L} (f, \eps) \approx (1/\eps)^{d(1-1/\nu)} \, c(\beta \eps)$ for small $\eps$. Furthermore, as noted by \cite{bouttier2020regret} in the single-fidelity setting, a Lipschitz function $f$ can feature different shapes at different scales. For example, $f(x)=1-\max\{|x|-1,0\}$ is constant in a neighborhood of the origin but equal to $f_2$ (up to an affine transformation to preserve continuity) outside of this neighborhood, with a non-unique maximizer at the origin. In this case, $S_{\beta, L}(f, \eps)$ behaves as $c(\beta \eps)$ (up to a log factor) for ``moderate'' values of $\eps$, but as $\left(\frac{L}{\eps}\right)^d c(\beta \eps)$ for smaller values of $\eps$. Many other examples admit such a multi-scale complexity behavior.

\section{Missing proofs in Section \ref{sec:lower_bound}}
\label{sec:missing_proofs}
\begin{proof}[Proof of Lemma \ref{lemma:certificate}]
  Let $\tau \in \N^*$. For the sake of clarity, we explicitly write the dependencies of the iterates $x_t(E), \alpha_t(E), y_t(E), x_t^*(E), \xi_t(E)$ w.r.t. the environment $E\in \mathcal{E}(f)$. Recall that $\mathcal{F}_L$ denotes the set of all $L$-Lipschitz functions from $\mathcal{X}$ to $\R$. Let $g \in \mathcal{F}_L$ be such that
  \begin{equation}
    \label{eq:conditiong}
    \forall t \leq \tau, g(x_t(E)) \in [y_t(E) - \alpha_t(E), y_t(E) + \alpha_t(E)] \;.
  \end{equation}
  Then there exists an environment $E^g \in \mathcal{E}(g)$ whose interactions with algorithm $A$ yield the same decisions $x_t$, $\alpha_t$, $x_t^*, \xi_t$ and observations $y_t$ as those generated with environment $E$, up to time $t=\tau$.
  More formally, we define $E^g = (E^g_t)_{t \geq 1}$ by $E^g_t(x, \alpha) = y_t(E)$ if $t \leq \tau$, $x=x_t(E)$ and  $\alpha=\alpha_t(E)$, but $E^g_t(x,\alpha) = g(x)$ otherwise.
  Note that $E^g \in \mathcal{E}(g)$ by \eqref{eq:conditiong}.\\
  First note that $x_1(E^g) = x_1(E)$ and $\alpha_1(E^g) = \alpha_1(E)$ since both terms are independent of the environment.
  From our definition of $E^g$, this implies that the approximation $y_1(E^g) = E_1^g(x_1(E^g), \alpha_1(E^g))$ returned by $E^g$ is equal to $y_1(E)$.
  Because the observation $y_1$ that $A$ receives is the same as before, $A$ outputs the same values for  $x^*_1, \xi_1, x_2$ and $\alpha_2$, which again implies that $y_2(E^g) = y_2(E)$.
  By a simple induction argument, we then have that, for all $t=1,\ldots,\tau$, $x_t(E^g) = x_t(E)$, $\alpha_t(E^g) = \alpha_t(E)$, $y_t(E^g) = y_t(E)$, $x^*_t(E^g) = x_t^*(E)$ and $\xi_t(E^g) = \xi_t(E)$. In particular,
  \begin{align*}
    \xi_\tau(E) = \xi_\tau(E_g)
    &\geq \max(g) - g(x^*_\tau(E_g)) = \max(g) - g(x^*_\tau(E)) \;,
  \end{align*}
  where the inequality follows from the definition of a certificate. Since the above lower bound on $\xi_\tau(E)$ is true for all $g \in \mathcal{F}_L$ satisfying condition \eqref{eq:conditiong}, it is also true for their supremum, which proves that $\xi_\tau(E) \geq \errt{A}$.
\end{proof}

\begin{proof}[Proof of Lemma \ref{lemma:cinf}]
  Let $A$ be a certified algorithm.
  We can assume without loss of generality that $\sup_{E \in \mathcal{E}(f)} \sprec < + \infty$.
  Let $C > \sup_{E \in \mathcal{E}(f)} \sprec$.
  Then, for any $E \in \mathcal{E}(f)$, by definition of $\sprec$, there exists $\tau \in \N^*$ such that $\sum_{t=1}^{\tau} c(\alpha_t(E)) \leq C$ and $\xi_\tau(E) \leq \eps$.
  By Lemma \ref{lemma:certificate} and since $A$ is a certified algorithm, this entails
  \[\sum_{t=1}^{\tau} c(\alpha_t(E)) \leq C \text{ and } \errt{A} \leq \eps.\]\\
  By definition of $\cinf$, this immediately yields $\cinf \leq C$.
  We conclude the proof by letting $C$ go to $\sup_{E \in \mathcal{E}(f)} \sprec$.
\end{proof}

In order to prove Lemma \ref{lemma:costeps0}, we first need the following intuitive lemma.

\begin{lemma}
  \label{lemma:errt-alpha}
  Let $f$ be any $L$-Lipschitz function, $\tau \in \N^*$, and let $E^*$ be the ``noiseless'' environment $E^* = ((x,\alpha) \mapsto  f(x))_{t \geq 1} \in \mathcal{E}(f)$.
  Then, for any certified algorithm $A$, the best possible certificate $\errtzero$ against $E^*$ is bounded from below by $\min\left\{\min_{t \leq \tau} \alpha_t(E^*), \eps_0/2\right\}$.
\end{lemma}

\begin{proof}
  Within this proof, we only work with the environment $E^*$, so we skip all dependencies of $x_t, \alpha_t, y_t, x^*_t, \xi_t$ on $E^*$.
  We set $\teps = \min\left\{\min_{t \leq \tau} \alpha_t, \eps_0/2\right\}$ and define $g:\mathcal{X} \to \R$ by $g(x) = \min\{ f(x) + \teps, f(x^*_\tau) - \teps + L\norm{x - x^*_\tau}\}$.
  We will now show that $g$ is compatible with the observations $(y_t)_{t \leq \tau}$ and the accuracies $(\alpha_t)_{t \leq \tau}$, and that $\max(g) - g(x^*_\tau) \geq \teps$.

  Since $g$ is the minimum of two $L$-Lipschitz functions, it is also $L$-Lipschitz. Moreover, for any $t \leq \tau$, on the one hand, $g(x_t) \leq f(x_t) +\teps \leq y_t + \alpha_t$, by definition of $g$ and $\teps$, and the fact that $y_t=f(x_t)$ (recall that we work with the ``noiseless'' environment $E^*$).
  
  On the other hand, by $L$-Lipschitz continuity of $f$, for all $x \in \mathcal{X}$, $L \norm{x - x^*_\tau} + f(x^*_\tau) - \teps \geq f(x) - \teps$.
  This implies that for all $x \in \mathcal{X}$, $g(x) \geq \min\{f(x) + \teps, f(x) - \teps\} = f(x) - \teps$.
  In particular, for all $t\leq \tau$,
  \[ g(x_t) \geq f(x_t) - \teps  \geq y_t - \alpha_t \;. \]
  To sum up, $g$ is an $L$-Lipschitz function such that $\abs{g(x_t) - y_t} \leq \alpha_t$ for all $1 \leq t \leq \tau$: the algorithm $A$ cannot make the difference between the functions $f$ and $g$.

  Now, let us bound $\max(g) - g(x^*_\tau)$ from below to derive a lower bound on $\errtzero$.\\
  First, $g(x^*_\tau) = \min\{f(x^*_\tau)+\teps, f(x^*_\tau) - \teps\} = f(x^*_\tau) - \teps$.
  Second, let $v \in \R^d$ be such that $\norm{v}_2 = 1$ and $x^*_\tau + \frac{\teps v}{L} \in \mathcal{X}$. Such a $v$ exists by  Lemma \ref{lemma:existsv} and the facts that $\teps \leq \eps_0/2 = L\cdot \diam(\mathcal{X})/2$ and $\mathcal{X}$ is compact and connected.
  Then,
  \[ f\left(x^*_\tau + \frac{\teps v}{L}\right) +\teps \geq f(x^*_\tau) - L \norm{\frac{\teps v}{L}}+\teps = f(x_\tau^*)  \]
  and
  \[f(x^*_\tau) - \teps + L\norm{x^*_\tau + \frac{\teps v}{L} - x_\tau^*} = f(x^*_\tau) - \teps + L\norm{\frac{\teps v}{L}} = f(x^*_\tau)\;. \]
  By definition of $g$, this entails
  \[f(x_\tau^*) \leq g\left(x^*_\tau + \frac{\teps v}{L}\right) \leq \max(g) \]
  Putting everything together, we get
  \[\errtzero \geq \max(g) - g(x^*_\tau) \geq f(x^*_\tau) - (f(x^*_\tau) - \teps) = \teps = \min\left\{\min_{t \leq \tau} \alpha_t,\eps_0/2\right\}, \] which concludes the proof.
\end{proof}

\begin{proof}[Proof of Lemma \ref{lemma:costeps0}]
  Recall the definition of $\cinf$, and let $C \in \R$ and $A' \in \mathcal{A}$ be such that for all $E \in \mathcal{E}(f)$, there exists $\tau \in \N^*$ such that $\sum_{t=1}^\tau c(\alpha_t(E)) \leq C$ and $\errt{A'}\leq \eps$.
  In particular, for the ``noiseless'' environment $E^*=((x, \alpha) \mapsto f(x))_{t \geq 1}$, there exists $\tau$ such that $\sum_{t = 1}^\tau c(\alpha_t(E^*)) \leq C$ and $\errtzero \leq \eps$.
  Using Lemma \ref{lemma:errt-alpha}, we get that $\min\left\{\min_{t \leq \tau} \alpha_t(E^*), \eps_0/2\right\} \leq \errtzero \leq \eps$, which implies $\min_{t \leq \tau} \alpha_t(E^*) \leq \eps$ (since $\eps_0/2>\eps$ by assumption).
  Because $c$ is non-negative and non-increasing, we have:
  \[ C \geq \sum_{t=1}^\tau c(\alpha_t(E^*)) \geq c\left(\min_{t\leq \tau} \alpha_t(E^*)\right) \geq c(\eps) \;.\]
  By definition of $\cinf$, this entails that $\cinf \geq c(\eps)$.
\end{proof}

\section{Useful lemmas}
\label{sec:proofs}

We recall two rather classical lemmas, and provide proofs for the convenience of the reader.

\begin{lemma}
  \label{lemma:existsv}
  Let $\mathcal{X} \subset \R^d$ be a compact connected set with diameter $\rho$.
  Then for any $x \in \mathcal{X}$ and $\eps \leq \frac{\rho}{2}$, there exists $v \in \R^d$ with $\norm{v} = 1$ such that $x + \eps v \in \mathcal{X}$.
\end{lemma}

\begin{proof}
    Let $x \in \mathcal{X}$ and $0 < \eps \leq \frac{\rho}{2}$ (the result is straightforward if $\eps=0$). Since $\rho$ is the diameter of $\mathcal{X}$ and the latter is compact, there exist $y, z \in \mathcal{X}$ such that $\norm{y-z} = \rho$.\\
  Let us show by contradiction that there exists a point $x' \in \mathcal{X}$ for which $\norm{x-x'}\geq \eps$.
  Assume for a moment that for all $x' \in \mathcal{X}$, $\norm{x-x'} < \eps$.
  In that case, we would have $\norm{y-z} \leq \norm{y-x} + \norm{x-z} < 2\eps \leq \rho$, which is in contradiction with $\norm{y-z} = \rho$.

  Now that we know that it exists, let $x' \in \mathcal{X}$ be such that $\norm{x-x'} \geq \eps$. Because $\mathcal{X}$ is connected, there exists a continuous path $\gamma :t \in [0, 1] \mapsto \gamma(t) \in \mathcal{X}$ such that $\gamma(0) = x$ and $\gamma(1) = x'$.
  Let $g$ be the function $g: t\in [0, 1] \mapsto \norm{\gamma(t) - x} \in \R$.
  $g$ is a continuous function from $[0, 1]$ to $\R$, with $g(0) = 0$ and $g(1) \geq \eps>0$, so according to the intermediate value theorem, there exists $t^* \in [0, 1]$ such that $g(t^*) = \eps$.
  Taking $x'' = \gamma(t^*)$ and $v = \frac{x''-x}{\norm{x-x''}}$ solves the problem, because $x + \eps v = x'' \in \mathcal{X}$.
\end{proof}

\begin{lemma}
  \label{lemma:packing_number}
  For any bounded set $E \subset \R^d$, and all  $0 < r_1 < r_2$, we have
  \[ \mathcal{N}(E, r_1) \leq \left(1+2\frac{r_2}{r_1}\right)^d \mathcal{N}(E, r_2).\]
\end{lemma}

The above lemma is well known and can be found, e.g., in \cite[Appendix A]{bachoc2021instance} with a slightly weaker statement.
We recall the proof for the convenience of the reader. For any $x \in \R^d$ and $r >0$, we set $B(x,r) = \{u \in \R^d: \norm{u-x} \leq r \}$.

\begin{proof}
  Fix any bounded set $E \subset \R^d$ and $0 < r_1 < r_2$.
  Consider an $r_1$-packing $F = \{x_1, \ldots, x_{N_1}\}$ of $E$, with cardinality $N_1 \defeq \mathcal{N}(E, r_1)$.\\
  Let $F_0 = F$.
  We define a sequence $F_0, F_1, \ldots, F_{k_{\text{end}-1}}$ of subsets of $F$ by induction, as follows.
  For $k \geq 1$ let $\hat{x}_k$ be any element of $F_{k-1}$, and define $B_k = F_{k-1} \cap B(\hat{x}_k, r_2)$ and $F_k = F_{k-1}\backslash B_k$.
  Repeating this procedure, we get an index $k_{\text{end}}\leq N_1$ such that $F_{k_{\text{end}}-1}$ is non-empty while $F_{k_{\text{end}}}$ is empty.\\
  Then, the set $\{\hat{x}_1, \ldots, \hat{x}_{k_\text{end}}\}$ is an $r_2$-packing of $E$, so that $k_{\text{end}} \leq \mathcal{N}(E, r_2)$.
  Let us now upper bound $N_1$ using $k_{\text{end}}$.
  By construction, the union of the $B_k$'s contains $F$, so for all $i \leq N_1$, there exists $k \leq k_{\text{end}}$ such that $x_i \in B_k$, and thus $B(x_i, r_1/2) \subset B(\hat{x}_k, r_2+r_1/2)$.
  Therefore,
  \[\bigcup_{1\leq i\leq N_1} B(x_i, r_1/2) \subset \bigcup_{1 \leq k \leq k_{\text{end}}} B(\hat{x}_k, r_2+r_1/2)\;.\]
  Moreover, the $N_1$  balls $B(x_i, r_1/2)$ are pairwise disjoint, because $F$ is an $r_1$-packing of $E$.
  By a volumetric argument, we thus get that $\left(r_1/2\right)^d N_1 \leq \left(r_2+r_1/2\right)^d k_{\text{end}} \leq \left(r_2+r_1/2\right)^d \mathcal{N}(E, r_2)$. Rearranging terms concludes the proof.
\end{proof}

\section{The special case of constant costs}
\label{sec:specialcase-constantcost}

In this short section we focus on the case of a constant cost $c(\alpha) = 1$. We formalize and prove the following very intuitive fact: when more accurate evaluations come at no additional cost, choosing the best accuracy available is always optimal. To avoid boundary effects, we slightly extend the setting by allowing the learner to choose $\alpha_t$ identically equal to zero. 

\begin{lemma}
  \label{lemma:no_prec}
  Let $f:\mathcal{X} \to \R$ be an $L$-Lipschitz function, and assume $c:\R_+ \to \R_+$ is the constant function given by $c(\alpha)=1$ for all $\alpha\geq 0$. Then,
  \begin{itemize}
      \item $\sprec = \inf \{ t \in \N^* \sep \xi_t(E) \leq \eps\}$;
      \item the smallest cost complexity (against the worst environment) is achieved by certified algorithms that choose $\alpha_t = 0$ for all $t \in \N^*$.
  \end{itemize}
\end{lemma}

\begin{proof}[Proof of Lemma \ref{lemma:no_prec}]
  Let $f$ be an $L$-Lipschitz function, and $\eps>0$. Since $c(\alpha) = 1$ for all $\alpha \geq 0$, the definition of $\sprec$ directly yields $\sprec = \inf\{t \in \N^* \sep \xi_t(E) \leq \eps\}$.
  
  Recall that $\mathcal{A}$ denotes the set of all certified algorithms, and that $\mathcal{E}(f)$ is the set of all environments associated with the function $f$.
  As discussed in the introduction, $\alpha_t$ depends on $E$ only via the inaccurate approximations $y_1, \ldots, y_{t-1}$ of $f$.
  We write $\alpha_t(y_1, \ldots, y_{t-1})$ instead of $\alpha_t$ to clarify this dependency when needed.
  We denote by $\mathcal{A}_0$ the set of all certified algorithms such that $\alpha_t(y_1, \ldots, y_{t-1}) = 0$ for all $t \in \N^*$ against all possible realizations of  $y_1, \ldots, y_{t-1}$.\\

\noindent
  Formally, what we want to prove is
  \[ \inf_{A' \in \mathcal{A}_0 }\sup_{E \in \mathcal{E}(f)} \sprecAE{A'}{E} \leq \inf_{A \in \mathcal{A} }\sup_{E \in \mathcal{E}(f)} \sprec\;. \]

  Let $A$ be an algorithm in $\mathcal{A}$, and let $E^* \in \mathcal{E}(f)$ be the environment for which $E^*_t(x, \alpha) = f(x)$ for all $t \in \N^*$, all $x \in \mathcal{X}$ and all $\alpha \geq 0$.
  Note that
  \begin{equation}
    \label{eq:supsprec}
    \sup_{E \in \mathcal{E}(f)} \sprec \geq \sprecAE{A}{E^*}
  \end{equation}

  Now, let $\tilde{A}$ be the same algorithm as $A$, with one difference: whatever the situation and the environment $E$, for all $t \in \N$, $\tilde{\alpha}_t(y_1, \ldots, y_{t-1}) = 0$.
  Then $\tilde{A} \in \mathcal{A}_0$.
  Moreover, because of the particularity of $E^*$, $A$ and $\tilde{A}$ will behave the same way against $E^*$: the sequences of query points $(x_t)_{t \in \N^*}$, recommendations $(x_t^*)_{t \in \N^*}$ and certificates $(\xi_t)_{t \in \N^*}$ are the same for $A$ and for $\tilde{A}$. From this, we get that $ \sprecAE{A}{E^*} = \sprecAE{\tilde{A}}{E^*}$.
  Combining with \eqref{eq:supsprec} yields
  \begin{equation}
    \inf_{A \in \mathcal{A}} \sup_{E \in \mathcal{E}(f)} \sprec \geq \inf_{A \in \mathcal{A}} \sprecAE{A}{E^*} \geq \inf_{\tilde{A} \in \mathcal{A}_0} \sprecAE{\tilde{A}}{E^*}
        \label{eq:sprectildeA}
  \end{equation}

  Let us now show that
  \begin{equation}
    \inf_{A' \in \mathcal{A}_0} \sprecAE{A'}{E^*} = \inf_{A' \in \mathcal{A}_0} \sup_{E \in \mathcal{E}(f)} \sprecAE{A'}{E},
    \label{eq:sprecEstar}
  \end{equation}
  which, combined with \eqref{eq:sprectildeA}, will conclude the proof.

  To see why \eqref{eq:sprecEstar} holds, let $A'$ be any certified algorithm with $\alpha_t=0$ for all $t \in  \N^*$, whatever the past observations $y_1, \ldots, y_{t-1}$, and let $E$, and $E'$ be two environments in $\mathcal{E}(f)$.
  Then, for all $x \in \mathcal{X}$, $E_t(x, 0) = E'_t(x, 0) = f(x)$, because environments should satisfy $\abs{E_t(x, \alpha) - f(x)} \leq \alpha$ for all $x \in \mathcal{X}$ and $\alpha \geq 0$.
  Therefore, the behavior of $A'$ against any $E \in \mathcal{E}(f)$ is the same as against $E^*$ which proves \eqref{eq:sprecEstar}.
\end{proof}

\bibliographystyle{siamplain}
\bibliography{biblio}

\begin{thebibliography}{10}

\bibitem{AuCeFi-02-FiniteTimeBandits}
{\sc P.~Auer, N.~Cesa-Bianchi, and P.~Fischer}, {\em Finite-time analysis of
  the multiarmed bandit problem}, {Mach. Learn.}, 47 (2002), pp.~235--256.

\bibitem{bachoc2021instance}
{\sc F.~Bachoc, T.~Cesari, and S.~Gerchinovitz}, {\em Instance-dependent bounds
  for zeroth-order {L}ipschitz optimization with error certificates}, Advances
  in Neural Information Processing Systems, 34 (2021), pp.~24180--24192.

\bibitem{bartlett2019simple}
{\sc P.~L. Bartlett, V.~Gabillon, and M.~Valko}, {\em A simple parameter-free
  and adaptive approach to optimization under a minimal local smoothness
  assumption}, in Algorithmic Learning Theory, PMLR, 2019, pp.~184--206.

\bibitem{bhat22-identifying_near_optimal_decis}
{\sc S.~P. Bhat and C.~Amballa}, {\em Identifying near-optimal decisions in
  linear-in-parameter bandit models with continuous decision sets}, in
  Proceedings of the Thirty-Eighth Conference on Uncertainty in Artificial
  Intelligence, J.~Cussens and K.~Zhang, eds., vol.~180 of Proceedings of
  Machine Learning Research, PMLR, 01--05 Aug 2022, pp.~181--190.

\bibitem{bonnans19-convexStochasticOptimization}
{\sc J.~F. Bonnans}, {\em Convex and Stochastic Optimization}, Universitext,
  Springer Cham, 2019.

\bibitem{bouttier2020regret}
{\sc C.~Bouttier, T.~Cesari, M.~Ducoffe, and S.~Gerchinovitz}, {\em Regret
  analysis of the {P}iyavskii-{S}hubert algorithm for global {L}ipschitz
  optimization}, 2020, \url{https://arxiv.org/abs/2002.02390v4}.

\bibitem{BoVa-04-ConvexOptimization}
{\sc S.~Boyd and L.~Vandenberghe}, {\em Convex optimization}, Cambridge
  University Press, 2004.

\bibitem{Bub-15-ConvexOptimization}
{\sc S.~Bubeck}, {\em Convex optimization: Algorithms and complexity},
  Foundations and Trends in Machine Learning, 8 (2015), pp.~231--357.

\bibitem{bubeck2011x}
{\sc S.~Bubeck, R.~Munos, G.~Stoltz, and C.~Szepesv{\'a}ri}, {\em X-armed
  bandits.}, Journal of Machine Learning Research, 12 (2011).

\bibitem{bubeck11-withoutLipschitzConstant}
{\sc S.~Bubeck, G.~Stoltz, and J.~Y. Yu}, {\em Lipschitz bandits without the
  {L}ipschitz constant}, in Algorithmic Learning Theory, J.~Kivinen,
  C.~Szepesv{\'a}ri, E.~Ukkonen, and T.~Zeugmann, eds., Berlin, Heidelberg,
  2011, Springer Berlin Heidelberg, pp.~144--158.

\bibitem{bull2011convergence}
{\sc A.~D. Bull}, {\em Convergence rates of efficient global optimization
  algorithms.}, Journal of Machine Learning Research, 12 (2011),
  pp.~2879--2904.

\bibitem{carlsson24-pureExplorationLinearConstraints}
{\sc E.~Carlsson, D.~Basu, F.~Johansson, and D.~Dubhashi}, {\em Pure
  exploration in bandits with linear constraints}, in Proceedings of The 27th
  International Conference on Artificial Intelligence and Statistics,
  S.~Dasgupta, S.~Mandt, and Y.~Li, eds., vol.~238 of Proceedings of Machine
  Learning Research, PMLR, 02--04 May 2024, pp.~334--342.

\bibitem{contal2013parallel}
{\sc E.~Contal, D.~Buffoni, A.~Robicquet, and N.~Vayatis}, {\em Parallel
  {G}aussian process optimization with upper confidence bound and pure
  exploration}, in Joint European Conference on Machine Learning and Knowledge
  Discovery in Databases, Springer, 2013, pp.~225--240.

\bibitem{danilin71-estimationEfficiencyAbsoluteMinimumFinding}
{\sc Y.~Danilin}, {\em Estimation of the efficiency of an
  absolute-minimum-finding algorithm}, USSR Computational Mathematics and
  Mathematical Physics, 11 (1971), pp.~261--267.

\bibitem{degenne19-nonAsymptoticPureExplorationSolvingGames}
{\sc R.~Degenne, W.~M. Koolen, and P.~M\'{e}nard}, {\em Non-asymptotic pure
  exploration by solving games}, in Advances in Neural Information Processing
  Systems, H.~Wallach, H.~Larochelle, A.~Beygelzimer, F.~d\textquotesingle
  Alch\'{e}-Buc, E.~Fox, and R.~Garnett, eds., vol.~32, Curran Associates,
  Inc., 2019.

\bibitem{degenne20-gamificationPureExplorationLinearBandits}
{\sc R.~Degenne, P.~Menard, X.~Shang, and M.~Valko}, {\em Gamification of pure
  exploration for linear bandits}, in Proceedings of the 37th International
  Conference on Machine Learning, H.~D. III and A.~Singh, eds., vol.~119 of
  Proceedings of Machine Learning Research, PMLR, 13--18 Jul 2020,
  pp.~2432--2442.

\bibitem{dewettinck1999modeling}
{\sc K.~Dewettinck, A.~De~Visscher, L.~Deroo, and A.~Huyghebaert}, {\em
  Modeling the steady-state thermodynamic operation point of top-spray
  fluidized bed processing}, Journal of Food Engineering, 39 (1999),
  pp.~131--143.

\bibitem{evendar02-PACboundsMAB}
{\sc E.~Even-Dar, S.~Mannor, and Y.~Mansour}, {\em {PAC} bounds for multi-armed
  bandit and {M}arkov {D}ecision {P}rocesses}, in Computational Learning
  Theory, J.~Kivinen and R.~H. Sloan, eds., Springer Berlin Heidelberg, 2002,
  pp.~255--270.

\bibitem{evendar06-actionElimination}
{\sc E.~Even-Dar, S.~Mannor, and Y.~Mansour}, {\em Action elimination and
  stopping conditions for the multi-armed bandit and reinforcement learning
  problems}, Journal of Machine Learning Research, 7 (2006), pp.~1079--1105.

\bibitem{fiegel2020adaptive}
{\sc C.~Fiegel, V.~Gabillon, and M.~Valko}, {\em Adaptive multi-fidelity
  optimization with fast learning rates}, in International Conference on
  Artificial Intelligence and Statistics, PMLR, 2020, pp.~3493--3502.

\bibitem{forrester2006design}
{\sc A.~I. Forrester, A.~J. Keane, and N.~W. Bressloff}, {\em Design and
  analysis of "noisy" computer experiments}, AIAA journal, 44 (2006),
  pp.~2331--2339.

\bibitem{forrester2007multi}
{\sc A.~I. Forrester, A.~S{\'o}bester, and A.~J. Keane}, {\em Multi-fidelity
  optimization via surrogate modelling}, Proceedings of the {R}oyal {S}ociety
  {A}: mathematical, physical and engineering sciences, 463 (2007),
  pp.~3251--3269.

\bibitem{GK16colt-OptimalBAI}
{\sc A.~Garivier and E.~Kaufmann}, {\em Optimal best arm identification with
  fixed confidence}, in 29th Annual Conference on Learning Theory, vol.~49 of
  Proceedings of Machine Learning Research, 2016, pp.~998--1027.

\bibitem{garivier21-nonAsymptoticSequentialTests}
{\sc A.~Garivier and E.~Kaufmann}, {\em Nonasymptotic sequential tests for
  overlapping hypotheses applied to near-optimal arm identification in bandit
  models}, Sequential Analysis, 40 (2021), pp.~61--96.

\bibitem{garnett2023bayesian}
{\sc R.~Garnett}, {\em Bayesian optimization}, Cambridge University Press,
  2023.

\bibitem{geiselhart2011integration}
{\sc K.~Geiselhart, L.~Ozoroski, J.~Fenbert, E.~Shields, and W.~Li}, {\em
  Integration of multifidelity multidisciplinary computer codes for design and
  analysis of supersonic aircraft}, in 49th AIAA Aerospace Sciences Meeting
  Including the New Horizons Forum and Aerospace Exposition, 2011, p.~465.

\bibitem{grill2015black}
{\sc J.-B. Grill, M.~Valko, and R.~Munos}, {\em Black-box optimization of noisy
  functions with unknown smoothness}, Advances in Neural Information Processing
  Systems, 28 (2015).

\bibitem{HansenETAL-91-NumberIterationsPiyavskii}
{\sc P.~Hansen, B.~Jaumard, and S.-H. Lu}, {\em On the number of iterations of
  {P}iyavskii's global optimization algorithm}, Mathematics of Operations
  Research, 16 (1991), pp.~334--350.

\bibitem{HansenETAL-92-Survey}
{\sc P.~Hansen, B.~Jaumard, and S.-H. Lu}, {\em Global optimization of
  univariate {L}ipschitz functions: {I}. survey and properties}, Mathematical
  Programming, 55 (1992), pp.~251--272.

\bibitem{HansenETAL-92-NewAlgorithms}
{\sc P.~Hansen, B.~Jaumard, and S.-H. Lu}, {\em Global optimization of
  univariate {L}ipschitz functions: {II}. new algorithms and computational
  comparison}, Mathematical Programming, 55 (1992), pp.~273--292.

\bibitem{horn06-UnknownLipschitz}
{\sc M.~Horn}, {\em Optimal algorithms for global optimization in case of
  unknown {L}ipschitz constant}, Journal of {C}omplexity, 22 (2006),
  pp.~50--70.

\bibitem{huang2006sequential}
{\sc D.~Huang, T.~T. Allen, W.~I. Notz, and R.~A. Miller}, {\em Sequential
  kriging optimization using multiple-fidelity evaluations}, Structural and
  Multidisciplinary Optimization, 32 (2006), pp.~369--382.

\bibitem{JaKa-17-NonConvexOptimizationML}
{\sc P.~Jain and P.~Kar}, {\em Non-convex optimization for machine learning},
  Foundations and Trends® in Machine Learning, 10 (2017), pp.~142--336.

\bibitem{jedra20-BAILinearBandits}
{\sc Y.~Jedra and A.~Proutiere}, {\em Optimal best-arm identification in linear
  bandits}, in Advances in Neural Information Processing Systems,
  H.~Larochelle, M.~Ranzato, R.~Hadsell, M.~Balcan, and H.~Lin, eds., vol.~33,
  Curran Associates, Inc., 2020, pp.~10007--10017.

\bibitem{jones1993lipschitzian}
{\sc D.~R. Jones, C.~D. Perttunen, and B.~E. Stuckman}, {\em Lipschitzian
  optimization without the {L}ipschitz constant}, Journal of Optimization
  Theory and Applications, 79 (1993), pp.~157--181.

\bibitem{jones1998efficient}
{\sc D.~R. Jones, M.~Schonlau, and W.~J. Welch}, {\em Efficient global
  optimization of expensive black-box functions}, Journal of Global
  optimization, 13 (1998), pp.~455--492.

\bibitem{jourdan22a-choosingAnswersEpsBAILinearBandits}
{\sc M.~Jourdan and R.~Degenne}, {\em Choosing answers in epsilon-best-answer
  identification for linear bandits}, in Proceedings of the 39th International
  Conference on Machine Learning, K.~Chaudhuri, S.~Jegelka, L.~Song,
  C.~Szepesvari, G.~Niu, and S.~Sabato, eds., vol.~162 of Proceedings of
  Machine Learning Research, PMLR, 17--23 Jul 2022, pp.~10384--10430.

\bibitem{kandasamy2019multi}
{\sc K.~Kandasamy, G.~Dasarathy, J.~Oliva, J.~Schneider, and B.~P{\'o}czos},
  {\em Multi-fidelity {G}aussian process bandit optimisation}, Journal of
  Artificial Intelligence Research, 66 (2019), pp.~151--196.

\bibitem{kandasamy2016gaussian}
{\sc K.~Kandasamy, G.~Dasarathy, J.~B. Oliva, J.~Schneider, and B.~P{\'o}czos},
  {\em Gaussian process bandit optimisation with multi-fidelity evaluations},
  Advances in Neural Information Processing Systems, 29 (2016), pp.~1000--1008.

\bibitem{kandasamy2016multi}
{\sc K.~Kandasamy, G.~Dasarathy, B.~P{\'o}czos, and J.~Schneider}, {\em The
  multi-fidelity multi-armed bandit}, Advances in Neural Information Processing
  Systems, 29 (2016), pp.~1777--1785.

\bibitem{kazerouni21-BAIGeneralizedLinearBandits}
{\sc A.~Kazerouni and L.~M. Wein}, {\em Best arm identification in generalized
  linear bandits}, {Operations Research Letters}, 49 (2021), pp.~365--371.

\bibitem{kleinberg2008multi}
{\sc R.~Kleinberg, A.~Slivkins, and E.~Upfal}, {\em Multi-armed bandits in
  metric spaces}, in Proceedings of the fortieth annual ACM symposium on Theory
  of computing, 2008, pp.~681--690.

\bibitem{kleinberg2019bandits}
{\sc R.~Kleinberg, A.~Slivkins, and E.~Upfal}, {\em Bandits and experts in
  metric spaces}, Journal of the ACM (JACM), 66 (2019), pp.~1--77.

\bibitem{kushner1964new}
{\sc H.~J. Kushner}, {\em A new method of locating the maximum point of an
  arbitrary multipeak curve in the presence of noise}, Journal of Basic
  Engineering, 86 (1964), pp.~97--106.

\bibitem{larson19-derivativeFreeOptimizationMethods}
{\sc J.~Larson, M.~Menickelly, and S.~M. Wild}, {\em Derivative-free
  optimization methods}, Acta Numerica, 28 (2019), p.~287–404.

\bibitem{LS20-banditalgos}
{\sc T.~Lattimore and C.~Szepesv\'{a}ri}, {\em Bandit algorithms}, Cambridge
  University Press, 2020.

\bibitem{gratiet2013multi}
{\sc L.~Le~Gratiet}, {\em Multi-fidelity Gaussian process regression for
  computer experiments}, PhD thesis, Universit{\'e} Paris-Diderot-Paris VII,
  2013.

\bibitem{locatelli2018adaptivity}
{\sc A.~Locatelli and A.~Carpentier}, {\em Adaptivity to smoothness in
  {X}-armed bandits}, in Conference on Learning Theory, PMLR, 2018,
  pp.~1463--1492.

\bibitem{malherbe2017global}
{\sc C.~Malherbe and N.~Vayatis}, {\em Global optimization of {L}ipschitz
  functions}, in International Conference on Machine Learning, PMLR, 2017,
  pp.~2314--2323.

\bibitem{mannor04-sampleComplexityExplorationMAB}
{\sc S.~Mannor and J.~N. Tsitsiklis}, {\em The sample complexity of exploration
  in the multi-armed bandit problem}, Journal of Machine Learning Research, 5
  (2004), pp.~623--648.

\bibitem{mockus75-bayesianMethodsExtremum}
{\sc J.~Mo{\v{c}}kus}, {\em On bayesian methods for seeking the extremum}, in
  Optimization Techniques IFIP Technical Conference Novosibirsk, G.~I. Marchuk,
  ed., Springer Berlin Heidelberg, 1975, pp.~400--404.

\bibitem{munos2011optimistic}
{\sc R.~Munos}, {\em Optimistic optimization of a deterministic function
  without the knowledge of its smoothness}, Advances in neural information
  processing systems, 24 (2011).

\bibitem{munos2014bandits}
{\sc R.~Munos}, {\em From bandits to {M}onte-{C}arlo {T}ree {S}earch: The
  optimistic principle applied to optimization and planning}, Foundations and
  Trends{\textregistered} in Machine Learning, 7 (2014), pp.~1--129.

\bibitem{Nes-04-ConvexOptimization}
{\sc Y.~Nesterov}, {\em Introductory lectures on convex optimization: a basic
  course}, vol.~87 of Applied Optimization, Springer US, 2004.

\bibitem{peherstorfer2018survey}
{\sc B.~Peherstorfer, K.~Willcox, and M.~Gunzburger}, {\em Survey of
  multifidelity methods in uncertainty propagation, inference, and
  optimization}, Siam Review, 60 (2018), pp.~550--591.

\bibitem{perevozchikov1990complexity}
{\sc A.~G. Perevozchikov}, {\em The complexity of the computation of the global
  extremum in a class of multi-extremum problems}, USSR Computational
  Mathematics and Mathematical Physics, 30 (1990), pp.~28--33.

\bibitem{picheny2010noisy}
{\sc V.~Picheny, D.~Ginsbourger, and Y.~Richet}, {\em {Noisy expected
  improvement and on-line computation time allocation for the optimization of
  simulators with tunable fidelity}}, in {2nd International Conference on
  Engineering Optimization}, 2010, p.~10 p.

\bibitem{picheny2013benchmark}
{\sc V.~Picheny, T.~Wagner, and D.~Ginsbourger}, {\em A benchmark of
  kriging-based infill criteria for noisy optimization}, Structural and
  multidisciplinary optimization, 48 (2013), pp.~607--626.

\bibitem{piyavskii1972algorithm}
{\sc S.~Piyavskii}, {\em An algorithm for finding the absolute extremum of a
  function}, USSR Computational Mathematics and Mathematical Physics, 12
  (1972), pp.~57--67.

\bibitem{qian08-bayesianHierarchicalModeling}
{\sc P.~Z.~G. Qian and C.~F.~J. Wu}, {\em Bayesian hierarchical modeling for
  integrating low-accuracy and high-accuracy experiments}, Technometrics, 50
  (2008), pp.~192--204.

\bibitem{sen2018multi}
{\sc R.~Sen, K.~Kandasamy, and S.~Shakkottai}, {\em Multi-fidelity black-box
  optimization with hierarchical partitions}, in International conference on
  machine learning, PMLR, 2018, pp.~4538--4547.

\bibitem{sen2019noisy}
{\sc R.~Sen, K.~Kandasamy, and S.~Shakkottai}, {\em Noisy blackbox optimization
  using multi-fidelity queries: A tree search approach}, in The 22nd
  international conference on artificial intelligence and statistics, PMLR,
  2019, pp.~2096--2105.

\bibitem{SKV19-GeneralParallelOptimizationWithoutMetric}
{\sc X.~Shang, E.~Kaufmann, and M.~Valko}, {\em General parallel optimization
  without a metric}, in Proceedings of the 30th International Conference on
  Algorithmic Learning Theory, A.~Garivier and S.~Kale, eds., vol.~98 of
  Proceedings of Machine Learning Research, PMLR, 22--24 Mar 2019,
  pp.~762--788.

\bibitem{shubert1972sequential}
{\sc B.~O. Shubert}, {\em A sequential method seeking the global maximum of a
  function}, SIAM Journal on Numerical Analysis, 9 (1972), pp.~379--388.

\bibitem{slivkins2019introduction}
{\sc A.~Slivkins}, {\em Introduction to multi-armed bandits}, Foundations and
  Trends{\textregistered} in Machine Learning, 12 (2019), pp.~1--286.

\bibitem{soare14-BAILinearBandits}
{\sc M.~Soare, A.~Lazaric, and R.~Munos}, {\em Best-arm identification in
  linear bandits}, in Advances in Neural Information Processing Systems,
  Z.~Ghahramani, M.~Welling, C.~Cortes, N.~Lawrence, and K.~Weinberger, eds.,
  vol.~27, Curran Associates, Inc., 2014, pp.~828--836.

\bibitem{song2019general}
{\sc J.~Song, Y.~Chen, and Y.~Yue}, {\em A general framework for multi-fidelity
  {B}ayesian optimization with {G}aussian processes}, in The 22nd International
  Conference on Artificial Intelligence and Statistics, PMLR, 2019,
  pp.~3158--3167.

\bibitem{Spa-03-StochasticSearchOptimization}
{\sc J.~C. Spall}, {\em Introduction to stochastic search and optimization:
  estimation, simulation, and control}, Wiley-Interscience, 2003.

\bibitem{srinivas2009gaussian}
{\sc N.~Srinivas, A.~Krause, S.~M. Kakade, and M.~W. Seeger}, {\em Gaussian
  process optimization in the bandit setting: No regret and experimental
  design}, in Proceedings of the 27th International Conference on Machine
  Learning, J.~F{\"{u}}rnkranz and T.~Joachims, eds., Omnipress, 2010,
  pp.~1015--1022.

\bibitem{SrinivasETAL-12-GPbandits}
{\sc N.~Srinivas, A.~Krause, S.~M. Kakade, and M.~W. Seeger}, {\em
  Information-theoretic regret bounds for {G}aussian process optimization in
  the bandit setting}, IEEE Transactions on information theory, 58 (2012),
  pp.~3088--3095.

\bibitem{sun2011multi}
{\sc G.~Sun, G.~Li, S.~Zhou, W.~Xu, X.~Yang, and Q.~Li}, {\em Multi-fidelity
  optimization for sheet metal forming process}, Structural and
  Multidisciplinary Optimization, 44 (2011), pp.~111--124.

\bibitem{valko2013stochastic}
{\sc M.~Valko, A.~Carpentier, and R.~Munos}, {\em Stochastic simultaneous
  optimistic optimization}, in International Conference on Machine Learning,
  PMLR, 2013, pp.~19--27.

\bibitem{Wainwright19-HighDimensionalStatistics}
{\sc M.~J. Wainwright}, {\em High-Dimensional Statistics: A Non-Asymptotic
  Viewpoint}, Cambridge Series in Statistical and Probabilistic Mathematics,
  Cambridge University Press, 2019.

\bibitem{wang2022procrastinated}
{\sc J.~Wang, D.~Basu, and I.~Trummer}, {\em Procrastinated tree search:
  Black-box optimization with delayed, noisy, and multi-fidelity feedback}, in
  Proceedings of the AAAI Conference on Artificial Intelligence, vol.~36, 2022,
  pp.~10381--10390.

\bibitem{wang21-fastPureExplorationFranckWolfe}
{\sc P.-A. Wang, R.-C. Tzeng, and A.~Proutiere}, {\em Fast pure exploration via
  {F}rank-{W}olfe}, in Advances in Neural Information Processing Systems,
  M.~Ranzato, A.~Beygelzimer, Y.~Dauphin, P.~Liang, and J.~W. Vaughan, eds.,
  vol.~34, Curran Associates, Inc., 2021, pp.~5810--5821.

\end{thebibliography}

\end{document}